%% file: main.tex
\documentclass{article} 
\usepackage{iclr2024_conference,times}

\input{math_commands.tex}

\usepackage{hyperref}
\usepackage{url}

\usepackage{upgreek}
\usepackage{pbox}
\usepackage[inline]{enumitem}
\usepackage{algorithm,algorithmic}
\usepackage{xspace}
\usepackage{amsthm,amsmath}
\usepackage{graphicx}
\usepackage{aliascnt}
\usepackage{cleveref}
\usepackage{autonum}
\usepackage{wrapfig}
\usepackage{subfig}
\usepackage{xspace}
\usepackage{stmaryrd}
\usepackage[inline]{enumitem}
\usepackage{url}

\usepackage{tikz}
\usetikzlibrary{calc}
\usepackage{pgfplots}
\usepackage{bbm}
\usepackage{ifthen}
\usepackage{xargs}
\usepackage[textwidth=1.8cm]{todonotes}

\crefname{theorem}{theorem}{Theorems}
\Crefname{Theorem}{Theorem}{Theorems}

\newtheorem*{lemma_nonumber*}{Lemma}

\newaliascnt{lemma}{theorem}
\newtheorem{lemma}[lemma]{Lemma}
\aliascntresetthe{lemma}
\crefname{lemma}{lemma}{lemmas}
\Crefname{Lemma}{Lemma}{Lemmas}

\newaliascnt{corollary}{theorem}

\aliascntresetthe{corollary}
\crefname{corollary}{corollary}{corollaries}
\Crefname{Corollary}{Corollary}{Corollaries}

\newaliascnt{proposition}{theorem}
\newtheorem{proposition}[proposition]{Proposition}
\aliascntresetthe{proposition}
\crefname{proposition}{proposition}{propositions}
\Crefname{Proposition}{Proposition}{Propositions}

\newaliascnt{definition}{theorem}
\newtheorem{definition}[definition]{Definition}
\aliascntresetthe{definition}
\crefname{definition}{definition}{definitions}
\Crefname{Definition}{Definition}{Definitions}

\newaliascnt{remark}{theorem}

\aliascntresetthe{remark}
\crefname{remark}{remark}{remarks}
\Crefname{Remark}{Remark}{Remarks}

\crefname{example}{example}{examples}
\Crefname{Example}{Example}{Examples}

\crefname{technique}{technique}{techniques}
\Crefname{Technique}{Technique}{Techniques}

\crefname{figure}{figure}{figures}
\Crefname{Figure}{Figure}{Figures}

\newtheorem{assumption}{\textbf{A}\hspace{-3pt}}
\crefformat{assumption}{{\textbf{A}}#2#1#3}

\newtheorem{assumptionF}{\textbf{F}\hspace{-3pt}}
\crefformat{assumptionF}{{\textbf{F}}#2#1#3}

\Crefname{assumptionB}{\textbf{B}\hspace{-3pt}}{\textbf{B}\hspace{-3pt}}
\crefname{assumptionB}{\textbf{B}}{\textbf{B}}

\Crefname{assumptionC}{\textbf{C}\hspace{-3pt}}{\textbf{C}\hspace{-3pt}}
\crefname{assumptionC}{\textbf{C}}{\textbf{C}}

\Crefname{assumptionH}{\textbf{H}\hspace{-3pt}}{\textbf{H}\hspace{-3pt}}
\crefname{assumptionH}{\textbf{H}}{\textbf{H}}

\Crefname{assumptionT}{\textbf{T}\hspace{-3pt}}{\textbf{T}\hspace{-3pt}}
\crefname{assumptionT}{\textbf{T}}{\textbf{T}}

\Crefname{assumptionT}{\textbf{T}\hspace{-3pt}}{\textbf{T}\hspace{-3pt}}
\crefname{assumptionT}{\textbf{T}}{\textbf{T}}

\Crefname{assumptionL}{\textbf{L}\hspace{-3pt}}{\textbf{L}\hspace{-3pt}}
\crefname{assumptionL}{\textbf{L}}{\textbf{L}}

\Crefname{assumptionQ}{\textbf{Q}\hspace{-3pt}}{\textbf{Q}\hspace{-3pt}}
\crefname{assumptionQ}{\textbf{Q}}{\textbf{Q}}

\Crefname{assumptionAR}{\textbf{AR}\hspace{-3pt}}{\textbf{AR}\hspace{-3pt}}
\crefname{assumptionAR}{\textbf{AR}}{\textbf{AR}}

\definecolor{MyRef}{HTML}{74787c}     

\definecolor{lightblue}{HTML}{9A8F97}
\definecolor{darkblue}{HTML}{1A254B}
\definecolor{blue}{HTML}{2B50AA}
\definecolor{CustomRed}{HTML}{cb1338}

\definecolor{DarkBlue}{HTML}{1A254B}
\definecolor{DarkRed}{HTML}{A4243B}
\colorlet{MyBlue}{blue}
\colorlet{MyRed}{CustomRed}

\input{def}

\title{Augmented bridge matching}

\author{Valentin De Bortoli \\
CNRS, ENS Ulm, France \\
\texttt{valentin.debortoli@gmail.com} 
\And
Guan-Horng Liu \\
Georgia Institute of Technology, USA \\
\AND
Tianrong Chen \\
Georgia Institute of Technology, USA
\And
Evangelos A. Theodorou \\
Georgia Institute of Technology, USA
\AND
Weilie Nie \\
NVIDIA
}
\begin{document}

\maketitle

\begin{abstract}
  Flow and bridge matching are a novel class of processes which encompass 
  diffusion models. One of the main aspect of their increased flexibility is
  that these models can interpolate between arbitrary data distributions,
  i.e.~they generalize beyond generative modeling and can be applied to
  learning stochastic (and deterministic) processes of arbitrary transfer tasks between two given distributions.
  In this paper, we highlight that while flow and bridge matching processes 
  preserve the information of the marginal distributions, they do \emph{not} necessarily 
  preserve the coupling information unless additional, stronger optimality conditions are met.
  This can be problematic if one aims at preserving the original empirical pairing. 
  We show that a
  simple modification of the matching process recovers this
  coupling by augmenting the velocity field (or drift) with the
  information of the initial sample point. Doing so, we lose the Markovian property of
  the process but preserve the coupling information between distributions. We illustrate the efficiency of our augmentation in learning mixture of image translation tasks.
\end{abstract}

\section{Introduction}

  Diffusion models \citep{song2019generative,ho2020denoising,sohl2015deep} have
  recently emerged as state-of-the-art models in generative modeling for diverse
  modalities: from text-to-images \citep{saharia2022photorealistic}, 3D
  \citep{poole2022dreamfusion}, video \citep{ho2022video} and protein synthesis
  \citep{watson2022broadly}. They rely on an iterative procedure, where the
  target data distribution is first corrupted using a \emph{forward noising
    process} converging to a Gaussian distribution. Then, the time-reversal of
  this process is learned leveraging techniques from score-matching
  \citep{hyvarinen2005estimation,vincent2011connection}.

  Diffusion models have recently been adapted to tackle general inverse
  problems. In that case, several frameworks have been proposed: from classifier
  and classifier-free guidance
  \citep{nichol2021beatgans,ho2022classifier,chung2022come,chung2022improving},
  with or without Sequential Monte Carlo correction \citep{trippe2022diffusion},
  to the replacement method in the context of inpainting
  \citep{lugmayr2022repaint} or Denoising Diffusion Restoration Models
  \citep{kawar2022denoising}. While these models have been successful they can be
  seen as guided modifications of diffusion models and in particular, at
  inference time, the output is still initialized with Gaussian noise. In the
  context of inverse problems where one observes a corrupted sample and tries to
  recover a clean version, it is desirable to instead start the inference from the
  corrupted sample. However, this requires changing the target distribution
  of the forward process which is often intractable and
  leads to numerical difficulties.

  To circumvent this issue, \citet{lipman2022flow,peluchettinon} have introduced
  a generalization of diffusion models called \emph{flow matching} (we call
  \emph{bridge matching} its stochastic counterpart) which interpolates between
  two distributions. In particular, both endpoints of the process are fixed by
  the user and the specific application. If the terminal distribution is given
  by the output of the forward noising process, then we recover diffusion
  models. However, bridge matching procedures are more flexible and recent works
  have leveraged them to solve inverse problems \citep{somnath2023aligned,guan2023i2isb,delbracio2023inversion,tong2023conditional,albergo2023stochastic,chung2023direct,heitz2023iterative}. In
  that case, we learn the velocity field of an Ordinary Differential Equation
  (ODE), or the drift of a Stochastic Differential Equation (SDE), between
  degraded and clean samples. To train such models, we assume that we have access to a \emph{coupled} dataset. This means that we can draw samples $(x_0,x_1)$ from a joint distribution. In particular, the training process assumes the knowledge of a certain  \emph{coupling} between the two data distributions.

  In this work, we answer the following question: does the flow/bridge matching
  procedure preserve the coupling information? We already know from
  \citet{gyongy1986mimicking} that it preserves the \emph{marginal}
  information. This is because both flow and bridge matching can be seen as
  specific instances of a \emph{Markovian projection} \citep{shi2023diffusion}
  which enjoys desirable properties. On the other hand, it has been shown 
  that the flow matching procedure provides a coupling which has a lower cost than
  the original one for every convex cost \citep{liu2022rectified}. As a first
  contribution, we show that the \emph{Markovian projection} in fact preserves
  the trajectories if and only if the original coupling is the entropic
  regularized optimal transport one, i.e. a static \emph{Schr\"odinger bridge}. Since in practice, available datasets of
  paired clean/corrupted samples are not ensured to satisfy this property it
  turns out that both bridge and flow matching \emph{do not preserve the
    coupling}. Preserving the coupling information is a desirable property especially in the context of inverse problems where the \emph{paired} dataset, and therefore the training coupling, encodes the relationship between the clean and the corrupted samples. 
    
    Next, we show that simply augmenting the drift (or
  velocity field) with the information of the initial position solves this
  problem, see \Cref{fig:augmented_bridge_matching_illus}. We call this
  procedure \emph{augmented bridge matching}. In order to show that these
  processes preserve the coupling information we leverage tools from the Doob
  $h$-transform literature \citep{palmowski2002technique,leonard2022feynman}.
  Using the notion of \emph{augmented bridge matching} we draw links with a
  recently introduced family of diffusion models \citep{zhou2023denoising} which
  also interpolate between arbitrary data distributions and show that augmented bridge
  matching and these models are two sides of the same coin. We illustrate the
  practical benefits of augmented bridge matching in synthetic examples as well
  as in image inverse problems.

  The rest of the paper is structured as follows. In \Cref{sec:background}, we
  recall the basics on bridge matching and the Doob $h$-transform which is a
  central tool to derive our \emph{augmented bridge matching} procedure. In
  \Cref{sec:sde-repr-mixt} we introduce our methodology and show that it
  preserves the coupling. We take a step back in \Cref{sec:bridge_matching_interpolation_diffusion} and
  show that in fact \emph{augmented bridge matching} can be defined in the
  context of diffusion models, drawing links with \citet{zhou2023denoising}. We
  discuss related work in \Cref{sec:related-work}. Finally, we showcase the
  efficiency of augmented bridge matching in \Cref{sec:experiments} in the context of image inverse problems and conclude
  in \Cref{sec:conclusion}.

\begin{figure}[h]
  \centering
  \includegraphics[width=\linewidth]{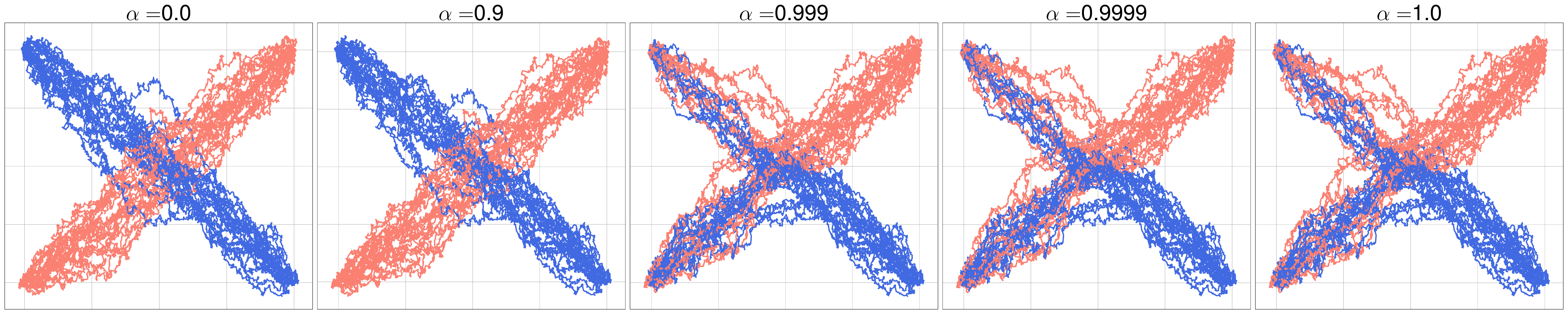}
  \caption{Influence of the augmentation on the bridge matching procedure in a toy example. In the training dataset the initial distribution corresponds to the Gaussian mixture with centers $(-2,-2)$ and $(-2,2)$. The terminal distribution corresponds to the Gaussian mixture with centers $(2,-2)$ and $(2,2)$. In the training dataset, the initial samples from the component with center $(-2,-2)$ are paired with the terminal ones from the component $(2,2)$ while the initial samples from the component with center $(-2,2)$ are paired with the terminal ones from the component $(2,-2)$. Each plot represents the learned coupling with a varying degree of augmentation ($\alpha=0$, augmented bridge matching and $\alpha=1$ original bridge matching, without augmentation). In blue we depict the trajectories arriving at the component $(2,-2)$ and in the red the ones arriving at $(2,2)$. While the training coupling is preserved for $\alpha =0$, its properties are lost as $\alpha$ increases from $0$ to $1$.}
  \label{fig:augmented_bridge_matching_illus}
\end{figure}

\textbf{Notation.} Given a probability space $\msx$, we denote
$\mathcal{P}(\msx)$ the space of probability measures on $\msx$. We denote by
$\pathmeas$, the space of \emph{path measures}, i.e.
$\pathmeas = \mathcal{P}(\rmC(\ccint{0,1}, \rset^d))$.  The subset of
\emph{Markov} path measures associated with an SDE of the form
$\rmd \bfX_t = v_t(\bfX_t) \rmd t + \sigma_t \rmd \bfB_t$, with $\sigma,v$
locally Lipschitz, is denoted $\calM$.  For any $\Qbb \in \calM$, we denote
$\Qbb_t$ its marginal distribution at time $t$, $\Qbb_{s,t}$ the joint
distribution at times $s$ and $t$, $\Qbb_{s|t}$ the conditional distribution at
time $s$ given state at time $t$, and $\Qbb_{|0,1} \in \pathmeas$ its
\emph{diffusion bridge}, i.e.~ the distribution of $\Qbb$ conditioned on
\emph{both} endpoints. Unless specified otherwise, all gradient operators
$\nabla$ are w.r.t. the variable $x_t$ with time index $t$. Let
$(\msx, \mathcal{X})$ and $(\mathsf{Y}, \mathcal{Y})$ be probability
spaces. Given a Markov kernel $\mathrm{K}: \ \msx \times \mathcal{Y} \to [0,1]$
and a probability measure $\mu$ defined on $\mathcal{X}$, we write
$\mu \mathrm{K}$ the probability measure on $\mathcal{Y}$ such that for any
$\mathsf{A} \in \mathcal{Y}$ we have
$\mu \mathrm{K}(\mathsf{A}) = \int_{\msx} \mathrm{K}(x, \mathsf{A}) \rmd
\mu(x)$.  In particular, for any joint distribution
$\Pi_{0,1} \in \mathcal{P}(\rset^d \times \rset^d)$, we denote the \emph{mixture
  of bridges} measure as $\Pi = \Pi_{0,1} \Qbb_{|0,1} \in \pathmeas$, which is
short for
$\Pi(\cdot) = \int_{\rset^d \times \rset^d} \Qbb_{|0,1}(\cdot|x_0, x_1)
\Pi_{0,1}(\rmd x_0, \rmd
x_1)$.  For any measure $\Pi \in \mathcal{P}(\rset^d)$ we
  define the \emph{entropy} of $\Pi$ has
  $\mathrm{H}(\Pi) = -\int_{\rset^d} \log(\rmd \Pi / \rmd \Leb)(x) \rmd \Pi(x)$
  if $\Pi$ admits a density w.r.t. the Lebesgue measure and $+\infty$
  otherwise.

\section{Background}
\label{sec:background}

\paragraph{Flow and Bridge matching.} We start by recalling some basics about bridge
matching, see \citet{lipman2022flow,peluchettinon} for instance.  We consider an initial
coupling $\Pi_{0,1}$, i.e. a probability measure on $\rset^d \times
\rset^d$. For example, in an imaging inverse problem, $\Pi_0$ represents the
distribution of \emph{corrupted} images and $\Pi_1$ the distribution of
\emph{clean} images. We also consider the \emph{path-measure} $\Qbb$ associated
with the Brownian motion $(\sigma \bfB_t)_{t \in \ccint{0,1}}$, with
$\sigma >0$\footnote{We recall basic elements of stochastic calculus in
  \Cref{sec:basics-stoch-calc}.}. Practically, speaking, a sample from $\Qbb$ is
a Brownian motion trajectory. We consider the path measure
$\Pbb = \Pi_{0,1} \Qbb_{|0,1}$. To sample from $\Pbb$, first sample
$(\bfX_0, \bfX_1) \sim \Pi_{0,1}$ and then sample a trajectory according to
$\Qbb_{|0,1}(\cdot|(\bfX_0, \bfX_1))$, i.e. sample a Brownian bridge
interpolating between $\bfX_0$ and $\bfX_1$. For any $t \in \ccint{0,1}$ we have
that $\bfX_t \sim \Pbb_t$ with
\begin{equation}
  \bfX_t = (1-t) \bfX_0 + t \bfX_1 + \sigma (t (1-t))^{1/2} \bfZ , \qquad \bfZ \sim \mathrm{N}(0, \Id) , \qquad (\bfX_0, \bfX_1) \sim \Pi_{0,1} . 
\end{equation}
This can also be expressed in a \emph{forward} fashion with
\begin{equation}
  \rmd \bfX_t = (\bfX_1 - \bfX_t) / (1 - t) \rmd t + \sigma \rmd \bfB_t , \qquad (\bfX_0, \bfX_1) \sim \Pi_{0,1} .
\end{equation}
Note that the measure $\Pbb$ is \emph{a priori} non-Markovian. In
\Cref{prop:comp-with-bridge}, we will show that the measure is Markovian if and
only if $\Pi_{0,1}$ is an entropic regularized Optimal Transport coupling
between $\Pi_0$ and $\Pi_1$, i.e. \emph{a static Schr\"odinger Bridge}. We are now ready to define the \emph{bridge
  matching} measure $\Pbb^\star$, associated with
$(\bfX_t^\star)_{t \in \ccint{0,1}}$ and given by
\begin{equation}
  \rmd \bfX_t^\star = (\mathbb{E}_{\Pbb_{1|t}}[\bfX_1] - \bfX_t^\star) / (1-t) \rmd t + \sigma \rmd \bfB_t , \qquad \bfX_0 \sim \Pi_0 .
\end{equation}
\begin{minipage}{0.4\textwidth}
  The conditional expectation $\mathbb{E}_{\Pbb_{1|t}}[\bfX_1]$ is estimated
  using that $\mathbb{E}_{\Pbb_{1|t}}[\bfX_1]$ is the minimizer of a $\ell_2$
  reconstruction loss. Note that in the case $\sigma = 0$, the stochastic term
  vanishes and we recover the flow matching framework of  \citet{lipman2022flow}.  In practice, we train a 
  neural
  network $\hat{x}_{1}^\theta$ to estimate 
  $\mathbb{E}_{\Pbb_{1|t}}[\bfX_1]$. The discrete Markov chain used at inference
  time is then obtained by discretizing
  $\rmd \bfX_t^\star = (\hat{x}_1^\theta(\bfX_t) - \bfX_t^\star) / (1-t) \rmd t
  + \sigma \rmd \bfB_t$ with $\bfX_0 \sim \Pi_0$. The full training algorithm is
  recalled in \Cref{alg:bridge_matching}.
\end{minipage}
\hfill
\begin{minipage}{0.58\textwidth}
  \vspace{-.75cm}
\begin{algorithm}[H]
\caption{Bridge Matching}
\label{alg:bridge_matching}
\begin{algorithmic}[1]
\STATE{\textbf{Input:} Joint distribution $\Pi_{0,1}$, Brownian bridge $\Qbb_{|0,1}$}
\STATE{Let $\Pbb = \Pi_{0,1} \Qbb_{|0,1}$.}
\WHILE{not converged}
\STATE{Sample $(\bfX_0, \bfX_1) \sim \Pi_{0,1}$}
\STATE{Sample $t \sim \mathrm{Unif}([0,1])$}
\STATE{Sample $\bfZ \sim \mathrm{N}(0, \Id)$}
\STATE{Sample $\bfX_t = (1-t) \bfX_0 + t \bfX_1 + \sigma (t(1-t))^{1/2} \bfZ$}
\STATE{ADAM step on $\| \hat{x}_1^\theta(\bfX_t) - \bfX_1 \|^2$}
  \ENDWHILE
  \STATE \textbf{Output:} $\hat{x}_1^{\theta^\star}$
\end{algorithmic}
\end{algorithm}
\end{minipage}

\paragraph{Schr\"odinger Bridges.} We briefly recall the concept of (static)
Schr\"odinger bridge \citep{schrodinger1932theorie}. Given a path measure
$\Qbb \in \mathcal{P}(\mathcal{C})$ and two distributions
$\Pi_0, \Pi_1 \in \mathcal{P}(\rset^d)$, the static Schr\"odinger bridge is defined as 
\begin{equation}
      \label{eq:sb_static}
    \Pi_{0,1}^\star = \argmin \ensembleLigne{\KLLigne{\mu_{0,1}}{\Qbb_{0,1}}}{\mu \in \mathcal{P}(\rset^d\times\rset^d), \ \mu_0=\Pi_0, \ \mu_1 = \Pi_1} . 
  \end{equation}
  In other words, $\Pi^\star_{0,1}$ is the closest measure to $\Qbb_{0,1}$
  w.r.t. the Kullback-Leibler divergence which satisfies the marginal
  constraints $\Pi_0$ and $\Pi_1$. In the case where $\Qbb$ is associated with
  $(\sigma \bfB_t)_{t \in \ccint{0,1}}$ with $\sigma > 0$, \eqref{eq:sb_static}
  can be rewritten as
\begin{equation}
      \label{eq:entropic_ot}
    \textstyle \Pi_{0,1}^\star = \argmin \ensembleLigne{\int \normLigne{x_0 - x_1}^2 \rmd \mu(x_0,x_1) - (2 \sigma^2) \mathrm{H}(\mu)}{\mu \in \mathcal{P}(\rset^d\times\rset^d), \ \mu_0=\Pi_0, \ \mu_1 = \Pi_1} ,
  \end{equation}
  where $\mathrm{H}$ is the entropy. Hence, in that setting, Schr\"odinger
  bridges coincide with the optimal coupling for the entropy regularized
  Wasserstein$-2$ distance. We refer to \citet{peyre2019computational} for more
  details on the links between Optimal Transport and Schr\"odinger Bridges. In our study we will show that the introduced augmented bridge matching procedure and the original bridge matching scheme coincide if and only if the training paired coupling is given by the static Schr\"odinger Bridge.

\paragraph{Doob $h$-transform.} 
The main theoretical tool of our analysis is the Doob $h$-transform
\citep{palmowski2002technique,leonard2022feynman}. Assume that an initial path
measure $\Pbb$ can be represented as an SDE, then the Doob $h$-transform theory implies
that some modification of $\Pbb$ at the terminal end point can also be expressed
using an SDE. More precisely, let
$\Qbb \in \mathcal{P}(\rmc(\ccint{0,1},\rset^d))$ be a path measure associated
with $\rmd \bfX_t = b_t(\bfX_t) \rmd t + \sigma_t \rmd \bfB_t$ with
$b_t: \ \rset^d \to \rset^d$ and $\sigma_t \geq 0$. Next, given a
\emph{potential function} $h_1: \ \rset^d \to \rset_+$, we define
$\Pbb \in \mathcal{P}(\rmc(\ccint{0,1},\rset^d))$ such that
for any $\omega \in \rmc(\ccint{0,1},\rset^d)$
\begin{equation}
  (\rmd \Pbb / \rmd \Qbb)(\omega) = h_1(\omega_1) .  \label{eq:twist}
\end{equation}
According to \eqref{eq:twist}, $\Pbb$ can be thought as a \emph{twisted} version
of $\Qbb$.  We also define
$h_t(x_t) = \int_{\rset^d} h_1(x_1) \Qbb_{1|t}(\rmd x_1, x_t)$.  Under mild
assumptions on $h_1$, $\Pbb$ is also associated with an SDE, see \cite[Lemma
3.1]{palmowski2002technique}, which we recall in the following proposition.

\begin{proposition}
  \label{prop:doob-h-transform}
  Assume that $h \in \rmc^2(\ccint{0,1} \times \rset^d, \rset_+)$ is bounded
  and that
  $(t,x) \mapsto \langle b_t(x) , \nabla h_t(x) \rangle + (\sigma_t^2/2) \Delta
  h_t(x)$ is measurable and bounded. Assume that
  $\inf \ensembleLigne{h_t(x)}{x \in \rset^d, t \in \ccint{0,1}} > 0$. Then,
  $\Pbb$ is associated with
  $\rmd \bfX_t = \{b_t(\bfX_t) + \sigma_t^2 \nabla \log h_t(\bfX_t) \} \rmd t +
   \sigma_t \rmd \bfB_t $.
\end{proposition}

This result will be key to establish the SDE representation of mixture of
bridges and to derive the \emph{augmented bridge matching} procedure. In
particular, using \Cref{prop:doob-h-transform}, we have that if
$\nabla \log h_t$ is available, then one can sample from $\Pbb$, the twisted version $\Qbb$, by first
sampling $\bfX_0 \sim \Pbb_0$ and then discretizing the dynamics given by 
$\rmd \bfX_t = \{b_t(\bfX_t) + \sigma_t^2 \nabla \log h_t(\bfX_t) \} \rmd t +
\sigma_t \rmd \bfB_t $.

\section{Augmented Bridge Matching}
\label{sec:sde-repr-mixt}

In this section, we first show that the fixed points of the  bridge matching procedure are
necessarily Schr\"odinger bridges. Then, we introduce Augmented Bridge Matching,
a simple modification of Bridge Matching which ensures that couplings are
preserved. As in \citet{guan2023i2isb,somnath2023aligned}, we consider a paired
setting with a coupling $\Pi_{0,1} \in \mathcal{P}(\rset^d \times \rset^d)$. We
aim at deriving a SDE representation of $\Pbb = \Pi_{0,1} \Qbb_{|0,1}$, where
$\Qbb_{|0,1}$ is a given bridge, usually a Brownian bridge. In what follows, we
assume that $\Qbb$ is associated with
$\rmd \bfX_t = b_t(\bfX_t) \rmd t + \sigma_t \rmd \bfB_t$ with
$b: \ \ccint{0,1} \times \rset^d \to \rset^d$ and
$\sigma: \ [0,1] \to (0,+\infty)$. We recover the Brownian case if $b=0$ and
$\sigma_t = \sigma > 0$.

\paragraph{Fixed points of bridge matching.} In the next proposition, we provide necessary and sufficient conditions under
which the coupling is preserved when considering bridge matching.  More
precisely, we consider $(\bfX_t^{\mathcal{M}})_{t \in \ccint{0,1}}$ given by
\begin{equation}
  \label{eq:markovian_projection} 
      \textstyle \rmd \bfX_t^{\mathcal{M}} = \{ b_t(\bfX_t^{\mathcal{M}}) + \sigma_t^2 \mathbb{E}_{\Pbb_{1|t}}[ \nabla \log \Qbb_{1|t}(\bfX_1^{\mathcal{M}}|\bfX_t^{\mathcal{M}}) \ | \ \bfX_t^{\mathcal{M}}] \} \rmd t + \sigma_t \rmd \bfB_t , \quad \bfX_0^{\mathcal{M}} \sim \Pi_0 .
    \end{equation}
    Importantly, note that in the case where $\Qbb$ is associated with
    $(\sigma \bfB_t)_{t \in \ccint{0,1}}$ then
    $(\bfX_t^{\mathcal{M}})_{t \in \ccint{0,1}}$ coincides with
    $(\bfX_t^\star)_{t \in \ccint{0,1}}$ in \Cref{sec:background}. In
    particular, \eqref{eq:markovian_projection} is a generalization of the
    bridge matching procedure considered in
    \citet{peluchettinon,guan2023i2isb,somnath2023aligned,lipman2022flow}, see \citet{shi2023diffusion} for a study of such bridges. We
    denote by $\Pbb^{\mathcal{M}}$ the (Markov) path measure associated with
    \eqref{eq:markovian_projection}. We recall that the (static) Schr\"odinger
    Bridge $\Pi_{0,1}^\star$ is given by \eqref{eq:sb_static}.  The next proposition is our main result.

    \begin{proposition}
      \label{prop:comp-with-bridge}
      Under mild assumptions on $\Pi$ and $\Qbb$,
      $\Pbb^{\mathcal{M}}_{0,1} = \Pi_{0,1}$ if and only if
      $\Pi_{0,1} = \Pi_{0,1}^\star$.
    \end{proposition}

    In other words, \Cref{prop:comp-with-bridge} shows that the bridge matching
    procedure preserves the coupling if and only if the original coupling used for training is the
    static Schr\"odinger Bridge. 
    In \citet{somnath2023aligned,guan2023i2isb},
    it remains debatable whether the training dataset pairs 
    $\{(X_0^i, X_1^i)\}_{i=1}^N$ inherently represent the solution to the static
    Schr\"odinger Bridge \eqref{eq:sb_static}.
    %
    In particular, this assumption is not ensured in real-world applications and can be easily violated even in low-dimensional cases as shown in \Cref{fig:augmented_bridge_matching_illus}.
    %
    Hence, \emph{the bridge matching
      procedure does not preserve the original coupling in general}.

    \paragraph{Comparison between stochastic and deterministic matching.} One key assumption of \Cref{prop:comp-with-bridge} is that $\sigma > 0$. In
    the case $\sigma =0$, we recover the \emph{flow matching} procedure
    \citep{lipman2022flow}, which is the deterministic counterpart of bridge
    matching. The fixed points of flow matching have been investigated in
    \citet{liu2022flow,liu2022rectified}. In particular, it is shown in
    \citet{liu2022flow} that any \emph{straight coupling} is a fixed point of
    flow matching. In particular it can be shown that there exists straight
    couplings which are not optimal coupling for any convex cost \citep[Example
    3.5]{liu2022rectified}. In contrast, \Cref{prop:comp-with-bridge} ensures
    that fixed points of bridge matching are always optimal couplings. This is a
    key difference between the deterministic and the stochastic setting. More precisely, by introducing stochasticity into the original process we can obtain the uniqueness of the fixed point of the bridge matching, which is not the case in the deterministic setting (flow matching).

\paragraph{Augmented Bridge Matching.}  We are now ready to introduce the Augmented Bridge Matching procedure.  We recall that $\Pbb = \Pi_{0,1} \Qbb_{|0,1}$, where $\Pi_{0,1}$ is the original training coupling and $\Qbb$ is associated with
\begin{equation}
\label{eq:Q_sde}
    \rmd \bfX_t = b_t(\bfX_t) \rmd t + \sigma_t \rmd \bfB_t , \qquad \bfX_0 \sim \Qbb_0 . 
\end{equation}
We define
$h_1 : \ \rset^d \times \rset^d \to \rset_+$ given for any $x_0,x_1$ by
$h_1(x_0,x_1) = (\rmd \Pbb_{1|0} / \rmd \Qbb_{1|0})(x_0,x_1)$. Similarly, as in
\Cref{sec:background}, we define for any $t \in \ccint{0,1}$ and
$x_0,x_t \in \rset^d$,
$h_t(x_0, x_t) = \int_{\rset^d} h_1(x_0,x_1) \Qbb_{1|t}(\rmd x_1, x_t)$. We have
the following result, which is a direct consequence of Doob's $h$-transform. The
full proof is postponed to \Cref{sec:proof-sde-repr}.

\begin{proposition}
  \label{prop:sde-repr-mixt}
  Assume that for any $x_0 \in \rset^d$, $h(x_0, \cdot)$ satisfies the
  conditions of \Cref{prop:doob-h-transform}. Then, $\Pbb$ is associated with 
  \begin{equation}
    \label{eq:augmented_bridge_matching}
    \textstyle \rmd \bfX_t = \{ b_t(\bfX_t) + \sigma_t^2 \mathbb{E}_{\Pbb_{1|0,t}}[ \nabla \log \Qbb_{1|t}(\bfX_1|\bfX_t) \ | \ \bfX_0, \bfX_t] \} \rmd t + \sigma_t \rmd \bfB_t , \quad \bfX_0 \sim \Pi_0 .
  \end{equation}
   In
  particular, $(\bfX_0, \bfX_1) \sim \Pi_{0,1}$.
\end{proposition}

\begin{proof}
  In what follows, we give an outline of the proof.  We fix $x_0 \in \rset^d$
  and for any $\omega \in \rmc(\ccint{0,1}, \rset^d)$ we have 
  $\rmd \Pbb_{|0}(\omega|x_0) = \rmd \Qbb_{|0}(\omega|x_0) (\rmd \Pbb_{1|0} / \rmd
  \Qbb_{1|0})(\omega_1, x_0)$. Then, using \Cref{prop:doob-h-transform} we obtain that  $\Pbb_{|0}(\cdot, x_0)$ is associated with
  \begin{equation}
    \label{eq:sde_rep}
    \textstyle \rmd \bfX_t = \{ b_t(\bfX_t) + \sigma_t^2 \nabla \log h_t(\bfX_t, x_0)\} \rmd t + \sigma_t \rmd \bfB_t . 
  \end{equation}
  Finally, we remark that for any $t \in \ccint{0,1}$ and $x_0, x_t \in \rset^d$ we have 
  \begin{align}
    &\textstyle \nabla \log h_{t}(x_t, x_0) \textstyle= \int_{\rset^d} \tfrac{\Pbb_{1|0}(x_1|x_0) \Qbb_{1|t}(x_1|x_t)}{\Qbb_{1|0}(x_1|x_0) h_{t}(x_t|x_0)} \nabla \log \Qbb_{1|t}(x_1|x_t) \rmd x_1 \\
                                                  &\qquad \textstyle= \int_{\rset^d} \tfrac{\Pbb_{1|0}(x_1|x_0) \Qbb_{t|0,1}(x_t|x_0,x_1)}{\Qbb_{t|0}(x_t|x_0) h_{t}(x_t|x_0)} \nabla \log \Qbb_{1|t}(x_1|x_t) \rmd x_1 \\
    &\qquad \textstyle= \int_{\rset^d} \tfrac{\Pbb_{t,1|0}(x_t,x_1|x_0)}{\Pbb_{t|0}(x_t|x_0)} \nabla \log \Qbb_{1|t}(x_1|x_t) \rmd x_1 = \textstyle  \int_{\rset^d} \nabla \log \Qbb_{1|t}(x_1|x_t) \Pbb_{1|t,0}(x_1|x_t,x_0) \rmd x_1 ,
  \end{align}
  where the first equality is obtained using the definition of $h_t$, the second using the key relationship that $\Qbb_{1|t} \Qbb_{t|0} = \Qbb_{t|0,1} \Qbb_{1|0}$ and the third using the definition of $\Pbb$ and that $\Pbb_{t|0} = \Qbb_{t|0} h_{t}$. The last equality follows from Bayes' rule.
\end{proof}

This suggests the following algorithm to approximately sample from $\Pi_{0,1}$:
first sample $\bfX_0 \sim \Pi_0$ and then sample $(\bfX_t)_{t \in \ccint{0,1}}$
according to \eqref{eq:sde_rep}. The term
$\int_{\rset^d} \nabla \log \Qbb_{1|t}(x_1|\bfX_t) \rmd \Pbb_{1|0,t}(x_1, (\bfX_0,
\bfX_t))$ is not tractable and is approximated by $v_t^\theta$ minimizing the
following loss function
\begin{equation}
  \textstyle \mathcal{L}(\theta) = \int_0^1 \lambda_t \expeLigne{\normLigne{v_t^\theta(\bfX_0, \bfX_t) -  \nabla \log \Qbb_{1|t}(\bfX_1|\bfX_t)}^2} \rmd \Pbb(\bfX_0, \bfX_t, \bfX_1) ,
\end{equation}
where $(\lambda_t)_{t \in \ccint{0,1}}$ is a weighting function. The full
algorithm is given in \Cref{alg:augmented_bridge_matching}.

\begin{minipage}{0.58\textwidth}
  \vspace{-.5cm}
\begin{algorithm}[H]
\caption{Augmented Bridge Matching (AugBM)}
\label{alg:augmented_bridge_matching}
\begin{algorithmic}[1]
\STATE{\textbf{Input:} Joint distribution $\Pi_{0,1}$, Brownian bridge $\Qbb_{|0,1}$}
\STATE{Let $\Pbb = \Pi_{0,1} \Qbb_{|0,1}$.}
\WHILE{not converged}
\STATE{Sample $(\bfX_0, \bfX_1) \sim \Pi_{0,1}$}
\STATE{Sample $t \sim \mathrm{Unif}([0,1])$}
\STATE{Sample $\bfZ \sim \mathrm{N}(0, \Id)$}
\STATE{Sample $\bfX_t = (1-t) \bfX_0 + t \bfX_1 + \sigma (t(1-t))^{1/2} \bfZ$}
\STATE{\textcolor{MyRed}{ADAM step on $\| \hat{x}_1^\theta(\bfX_0, \bfX_t) - \bfX_1 \|^2$}}
  \ENDWHILE
  \STATE \textbf{Output:} $x_1^{\theta^\star}$
\end{algorithmic}
\end{algorithm}
\end{minipage}
\hfill
\begin{minipage}{0.38\textwidth}
  \paragraph{Comparison with bridge matching.} The algorithm described in
  \Cref{alg:augmented_bridge_matching} closely resembles the ones proposed in
  \citet{guan2023i2isb,somnath2023aligned}, see
  \Cref{alg:bridge_matching}. However, the main difference (highlighted in
  \textcolor{MyRed}{red}) is that the velocity field $v_t$ can depends on $x_0$
  in our case. Doing so, we do not get that $(\bfX_t)_{t \in \ccint{0,1}}$ is
  Markovian as it depends on the initial condition $\bfX_0$. However, in
  \citet{guan2023i2isb,somnath2023aligned}, the coupling $\Pi_{0,1}$ is provably not
  preserved by $(\bfX_t)_{t \in \ccint{0,1}}$.
\end{minipage}

Therefore in Augmented Bridge Matching the drift depends on both $\bfX_0$ and $\bfX_t$, i.e. the dynamics is not Markovian, but the coupling is preserved. In the original bridge matching procedure, the drift only depends on $\bfX_t$ and the coupling is not preserved, as we have seen in \Cref{prop:comp-with-bridge}. It is also possible to interpolate between the fully augmented bridge matching procedure and the original bridge matching by conditioning the velocity field on $\bfX_{\alpha t}$. If $\alpha =0$ we recover the augmented bridge matching while for $\alpha =1$ we recover the original scheme. The effect of $\alpha$ on the preservation of the original coupling is illustrated in \Cref{fig:augmented_bridge_matching_illus}. 
\section{Bridge matching and diffusion model based interpolation}
\label{sec:bridge_matching_interpolation_diffusion}

Finally, in this section, we show that the bridge matching procedure can be
shown to be equivalent to a methodology based on diffusion models. We start by recalling
the method introduced in \citet{zhou2023denoising}, which leverages the
framework of diffusion models to propose an interpolation procedure.

\paragraph{Denoising Diffusion Bridge Models.}  As in
the previous section, we consider a paired setting with a coupling
$\Pi_{0,1} \in \mathcal{P}(\rset^d \times \rset^d)$. We aim at deriving a SDE
representation of $\Pbb = \Pi_{0,1} \Qbb_{|0,1}$. First, consider the forward
noising process given by $\Pbb_{|1}(\cdot|x_1) = \Pi_{0|1} \Qbb_{|0,1}(\cdot | x_1)$ for an arbitrary
$x_1 \in \rset^d$. This forward noising process is associated with
\begin{equation}
  \rmd \bfX_t = \{b_t(\bfX_t) + \sigma_t^2 \nabla \log \Qbb_{1|t}(x_1|\bfX_t)\} \rmd t + \sigma_t \rmd \bfB_t , \qquad \bfX_0 \sim \Pi_{0|1} . 
\end{equation}
This forward representation is associated with a backward one, $(\bfY_{1-t})_{t \in \ccint{0,1}} = (\bfX_t)_{t \in \ccint{0,1}}$ where
\begin{equation}
  \label{eq:backward_ddbm}
  \rmd \bfY_t = \{-b_{1-t}(\bfY_t) - \sigma_{1-t}^2 \nabla \log \Qbb_{1|1-t}(\bfY_1|\bfY_t) + \sigma_{1-t}^2 \nabla \log \Pbb_{1-t|1}(\bfY_t | \bfY_1) \} \rmd t + \sigma_{1-t} \rmd \bfB_t . 
\end{equation}
If we consider $(\bfY_t)_{t \in \ccint{0,1}}$ such that $\bfY_1 \sim \Pi_1$ and
$(\bfY_t)_{t \in \ccint{0,1}}$ satisfies \eqref{eq:backward_ddbm}, then we
have that $(\bfY_t)_{t \in \ccint{0,1}} \sim \Pi_1 \Pbb_{|1} = \Pi_{0,1} \Qbb_{|0,1}$. Hence, we have that $(\bfY_t)_{t \in \ccint{0,1}} \sim \Pbb$ and in particular the coupling is preserved, i.e. $(\bfY_0, \bfY_1) \sim \Pi_{0,1}$ as in Augmented Bridge Matching. While
$\nabla \log \Qbb_{1|1-t}$ is tractable, $\nabla \log \Pbb_{1-t|1}$ is not and
needs to be approximated. However, it can be shown that 
\begin{equation}
  \textstyle \nabla \log \Pbb_{1-t|1}(x_{1-t}|x_1) = \int_{\rset^d} \nabla \log \Qbb_{1-t|0,1}(x_{1-t}|x_0,x_1) \rmd \Pbb_{0|1-t,1}(x_0|x_{1-t},x_1) . 
\end{equation}
Hence, we get that
\begin{align}
  \textstyle \nabla \log \Pbb_{t|1}(x_{t}|x_1) = \argmin_s \mathbb{E}_{\Pi_{0,1}\Qbb_{t|0,1}}[ \| s(t, \bfX_{t}, \bfX_1) - \nabla \log \Qbb_{t|0,1}(\bfX_{t}|\bfX_0,\bfX_1) \|^2 ] ,
\end{align}
which can be approximately computed since $\Qbb_{t|0,1}$ is tractable, see
\cite[Theorem 1]{zhou2023denoising}. At inference time, \citet{zhou2023denoising}
sample from $\Pi_1$ and compute $(\bfY_t)_{t \in \ccint{0,1}}$ using
\eqref{eq:backward_ddbm} and the score approximation. 

\paragraph{Unifying the models.} We now show that DDBM is equivalent to the
\emph{augmented} bridge matching procedure we propose.  By definition
$\Qbb$ is associated with
$\rmd \bfX_t = b_t(\bfX_t) \rmd t + \sigma_t \rmd \bfB_t$ and
$\bfX_0 \sim \Qbb_0$, see \eqref{eq:Q_sde}. Using the time-reversal formula, we get that $\Qbb$ is also 
associated with $(\bfY_{1-t})_{t \in \ccint{0,1}}$ such that
\begin{equation}
  \label{eq:time_reversal_Qbb}
  \rmd \bfY_t = \{-b_{1-t}(\bfY_t) + \sigma_{1-t}^2 \nabla \log \Qbb_{1-t}(\bfY_t) \} \rmd t + \sigma_t \rmd \bfB_t , \qquad \bfY_1 \sim \Qbb_1 . 
\end{equation}
Therefore, in order to sample from $\Pi_1 \Qbb_{|1}$, first sample
$\bfY_1 \sim \Pi_1$ and then sample $(\bfY_t)_{t \in \ccint{0,1}}$ according to
\eqref{eq:time_reversal_Qbb}. Next, we apply \Cref{prop:sde-repr-mixt} with
$b_t$ replaced $-b_{1-t}+ \sigma_{1-t}^2 \nabla \log \Qbb_ {1-t}$, and $\Pi_0$
replaced by $\Pi_1$. The augmented bridge matching procedure of \Cref{prop:sde-repr-mixt} is given by $\bfY_0 \sim \Pi_1$ and
  \begin{align}
    \label{eq:augmented_bridge_matching_reversal}
    &\textstyle \rmd \bfY_t = \{ -b_{1-t}(\bfX_t) + \sigma_{1-t}^2 \nabla \log \Qbb_{1-t}(\bfY_t) \\ & \qquad \qquad \qquad +  \sigma_{1-t}^2 \mathbb{E}_{\Pbb_{0|1,1-t}}[ \nabla \log \Qbb_{0|1-t}(\bfY_1|\bfY_0) \ | \ \bfY_0, \bfY_t]) \} \rmd t + \sigma_{1-t} \rmd \bfB_t .
  \end{align}
  The next proposition shows that \eqref{eq:augmented_bridge_matching_reversal}
  and \eqref{eq:backward_ddbm} are the same SDEs.

  \begin{proposition}
    \label{prop:identity_nabla_log}
For any $t \in \ccint{0,1}$ we have $\nabla \log \Qbb_t + \mathbb{E}_{\Pbb_{0|1,t}}[\nabla \log \Qbb_{0|t}] = \nabla \log \Pbb_{t|1} - \nabla \log \Qbb_{1|t}$.
  \end{proposition}

  Combining \Cref{prop:identity_nabla_log},
  \eqref{eq:augmented_bridge_matching_reversal} and \eqref{eq:backward_ddbm} and
  we can conclude that DDBM is an augmented bridge matching procedure for the
  time-reversed SDE, i.e. the two methods coincide if we invert the roles of $\Pi_0$ and $\Pi_1$ and the
  arrow of time.

\section{Related work}
\label{sec:related-work}

\paragraph{Bridge matching.} Bridge matching was first introduced by
\citet{peluchettinon}. It has recently been applied to high dimensional settings
in
\citet{guan2023i2isb,somnath2023aligned,delbracio2023inversion,tong2023conditional,albergo2023stochastic,albergo2022building,pooladian2023multisample}. The
deterministic counterpart to this procedure was first introduced in
\citet{lipman2022flow} and named \emph{flow matching}, see also
\citet{heitz2023iterative}. These approaches first sample both the initial and
terminal point according to a coupling, i.e. they assume access to a
\emph{paired} dataset (which might be given by the independent coupling). Then,
they consider an \emph{interpolation} between these two samples using a Brownian
bridge (or a linear interpolation in the case of
\citet{lipman2022flow}). Finally, they learn a drift (or velocity field) which
best approximates this non-Markovian dynamics. It turns out that this procedure
is the \emph{Markovian projection} \citep{gyongy1986mimicking} which has been
studied in finance. We note that it is possible to extend this procedure to the
case where only \emph{unpaired} datasets are available has been considered in
using iterative procedure and tools from Optimal Transport
\citep{liu2022flow,liu2022rectified,shi2023diffusion,peluchetti2023diffusion,tong2023conditional}.

\paragraph{Diffusion Schr\"odinger Bridge. } The theoretical properties of
Schr\"odinger Bridges \citep{schrodinger1932theorie} have been extensively
investigated using tools from probability and stochastic control theory
\citep{leonard2014survey,dai1991stochastic,chen2020optimal}.  In the context of
machine learning, the Diffusion Schr\"odinger Bridge algorithm was introduced in
\citet{debortoli2021diffusion,vargas2021solving,chen2021likelihood}.  These
methods have been extended to solve conditional simulation and  more general
control problems
\citep{shi2022conditional,thornton2022riemannian,liu2022deep,tamir2023transport}. Recently
\citet{shi2023diffusion,peluchetti2023diffusion} have introduced a new iterative
methodology to solve Schr\"odinger bridges based on Markovian and reciprocal
class projections.

\paragraph{Doob $h$-transform.} The Doob $h$-transform is an ubiquitous tool in
probability theory
\citep{palmowski2002technique,rogers2000diffusions,chung2006markov}. In
the context of diffusion models it has been leveraged in
\citet{heng2021simulating,liu2022let} to learn bridges. Recently, machine learning approaches have
also been proposed to learn Doob $h$-transforms in the context of sampling
\citep{chopin2023computational}.

\section{Experiments}
\label{sec:experiments}
In \Cref{sec:toy_xp} we first empirically verify that the Schr\"odinger Bridge is the only fixed point of the bridge matching procedure in a Gaussian setting. In \Cref{sec:multi_domain}, we illustrate the efficiency of the augmented bridge matching procedure in a multi-domain image-to-image translation setting.

\subsection{Toy experiment}
\label{sec:toy_xp}
First, we illustrate
  \Cref{prop:comp-with-bridge} in a simple Gaussian setting. We consider the
  case where $\Pi_0 = \Pi_1 = \mathrm{N}(0,1)$. We assume that
  $\Pbb_{0,1} = \mathrm{N}(0, \Sigma^\alpha)$, with
  $\Sigma^\alpha \in \rset^{d \times d}$ symmetric with
  $\Sigma^\alpha_{0,0} = \Sigma^\alpha_{1,1} = 1$ and
  $\Sigma^\alpha_{0,1} = \alpha$ with $\alpha \in \ooint{0,1}$. We consider
  $\Qbb$ associated with $(\sigma \bfB_t)_{t \in \ccint{0,1}}$. We have the
  following result.

  \begin{proposition}
    \label{prop:gaussian_example}
    The static Schr\"odinger Bridge $\Pi_{0,1}^\star$, solution of
    \eqref{eq:sb_static} is given by $\Pi^\star_{0,1} = \mathrm{N}(0, \Sigma^\star)$
    with $\Sigma^\star = \Sigma^{\alpha^\star}$ with $\alpha^\star = (\sigma^2/2)(\sqrt{1+4/\sigma^4} - 1)$.
    In addition, for any $\alpha \in \ooint{0,1}$ and
    $\Pi_{0,1} = \mathrm{N}(0, \Sigma^\alpha)$, we get that
    $\Pbb_{0,1}^{\mathcal{M}} = \mathrm{N}(0, \Sigma^{f(\alpha)})$ with $f$ explicit and given in the appendix.
   \end{proposition}

   Finally, in \Cref{fig:gaussian_example}, we empirically illustrate that
   $f(\alpha) = \alpha$, i.e. $\Pi_{0,1} = \Pbb_{0,1}^{\mathcal{M}}$ only if
   $\alpha = \alpha^\star$ for different values of $\sigma$. This means that,
   even in this simplified Gaussian setting, the Markovian projection only
   preserves the coupling if and only if $\Pi_{0,1}$ is the Schr\"odinger
   Bridge. 

      \begin{wrapfigure}{r}{0.5\textwidth}
     \begin{center}
     \includegraphics[width=.6\linewidth]{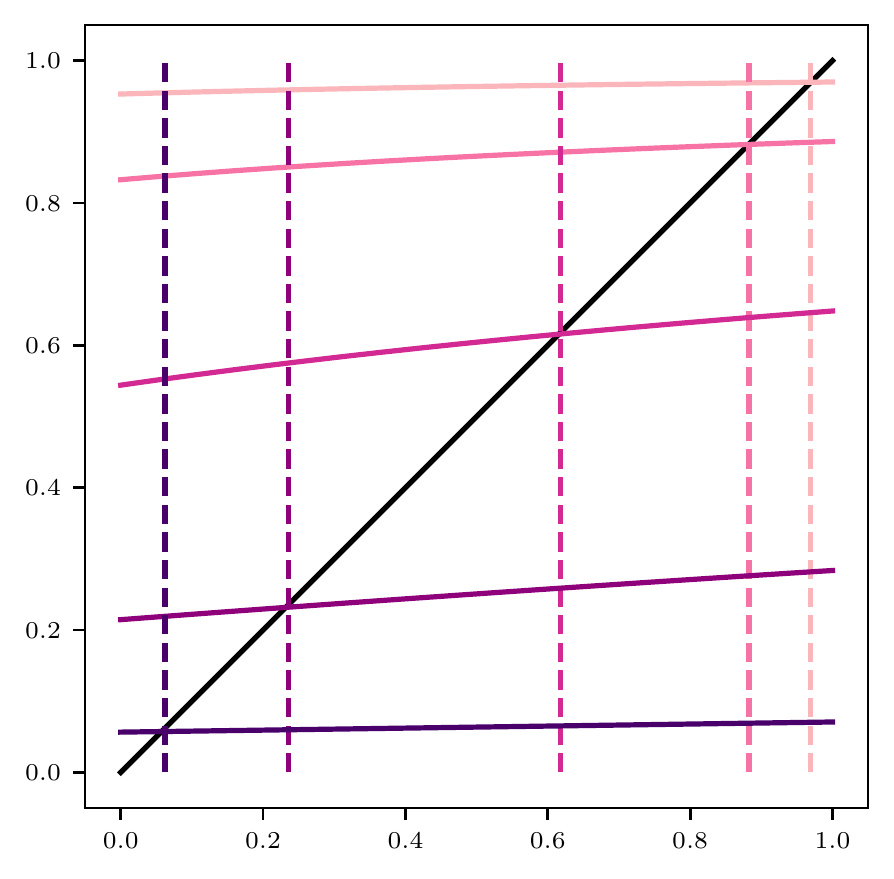}
     \caption{$f(\alpha)$ as a function of $\alpha$ for
       $\sigma \in \{0.25, 0.5, 1., 2.\}$ (dark to clear). Vertical lines
       correspond to the value of $\alpha^\star$ for the corresponding value of
       $\sigma$. Note that $f(\alpha) = \alpha$ if and only if $\alpha = \alpha^\star$, i.e. the coupling is preserved if and only if the original coupling is the Schr\"odinger Bridge.}
     \label{fig:gaussian_example}
     \end{center}
   \end{wrapfigure}

\subsection{Multi-domain image-to-image translation}
\label{sec:multi_domain}

Next, we test out our AugBM on multi-domain image-to-image translation.
Our goal is to train a single diffusion model with versatile capabilities, enabling it to perform translations across multiple domains.
Specifically, we consider two popular image-to-image translation tasks from pix2pix \citep{isola2017image}, namely \texttt{edges2shoes} and \texttt{edges2handbags}, and the colorization task on ImageNet dataset.
For each task, we train an AugBM to transfer between two provided domains in a bidirectional manner, without prior knowledge of which domain is in focus.
We compare our AugBM mainly to I$^2$SB \citep{guan2023i2isb}, which can be viewed as the special case of our AugBM without augmentation. 
To ensure a fair comparison, we adopt the same setup from I$^2$SB, where we 
parametrize {\color{red} $\hat{x}_1^\theta$} using a U-Net \citep{ronneberger2015u} initialized with the unconditional ADM checkpoint \citep{nichol2021beatgans}. We maintain consistency across all training hyperparameters to ensure that any observed performance disparities are solely attributed to the algorithmic distinctions between the models.
All images are in 256$\times$256 resolution.

\Cref{fig:pix2pix} reports the qualitative results, which encompass both the generation processes (in the odd rows) and the predicted couplings (in the even rows), i.e., {\color{red}$\bfX_T^\varepsilon$}.
It's worth noting a significant deviation in the generation processes between the two models beyond the midpoint, i.e., after $t\ge0.5$. In the case of our AugBM, it effectively transports samples to their correct target domains. In contrast, I$^2$SB tends to return them to their original source, effectively recreating the same source image. 
Indeed, for bidirectional translation tasks, it is extremely difficult to determine the target direction based merely on the intermediate samples, especially at the midpoint, without additional information, such as the augmentation utilized in AugBM.

It is essential to note that, in such cases, \emph{both models preserve the marginal, but only our AugBM preserves the coupling}. 
\Cref{fig:pix2pix_fid} provides quantitative demonstration,
where we report the Frechet Inception Distance (FID; \citep{fid}) w.r.t. the joint domain as a measure of marginal,
and the averaged FID w.r.t. individual domains---as a measure of coupling. 
All FID values are computed w.r.t. the validation statistics.\footnote{
    \texttt{edges2shoes} and \texttt{edges2handbags} entail their own validation sets.
    For ImageNet colorization, we use the same 10k ImageNet validation subset as in prior works \citep{saharia2022palette,guan2023i2isb}.
}
Although both models manage to achieve low marginal FIDs (with AugBM maintaining a performance advantage), our AugBM surpasses I$^2$SB when it comes to coupling FIDs, particularly in scenarios with high number of function evaluation (NFE)  regimes.
%

   \begin{figure}[h]
   \vskip -0.05in
     \centering
     \includegraphics[width=\linewidth]{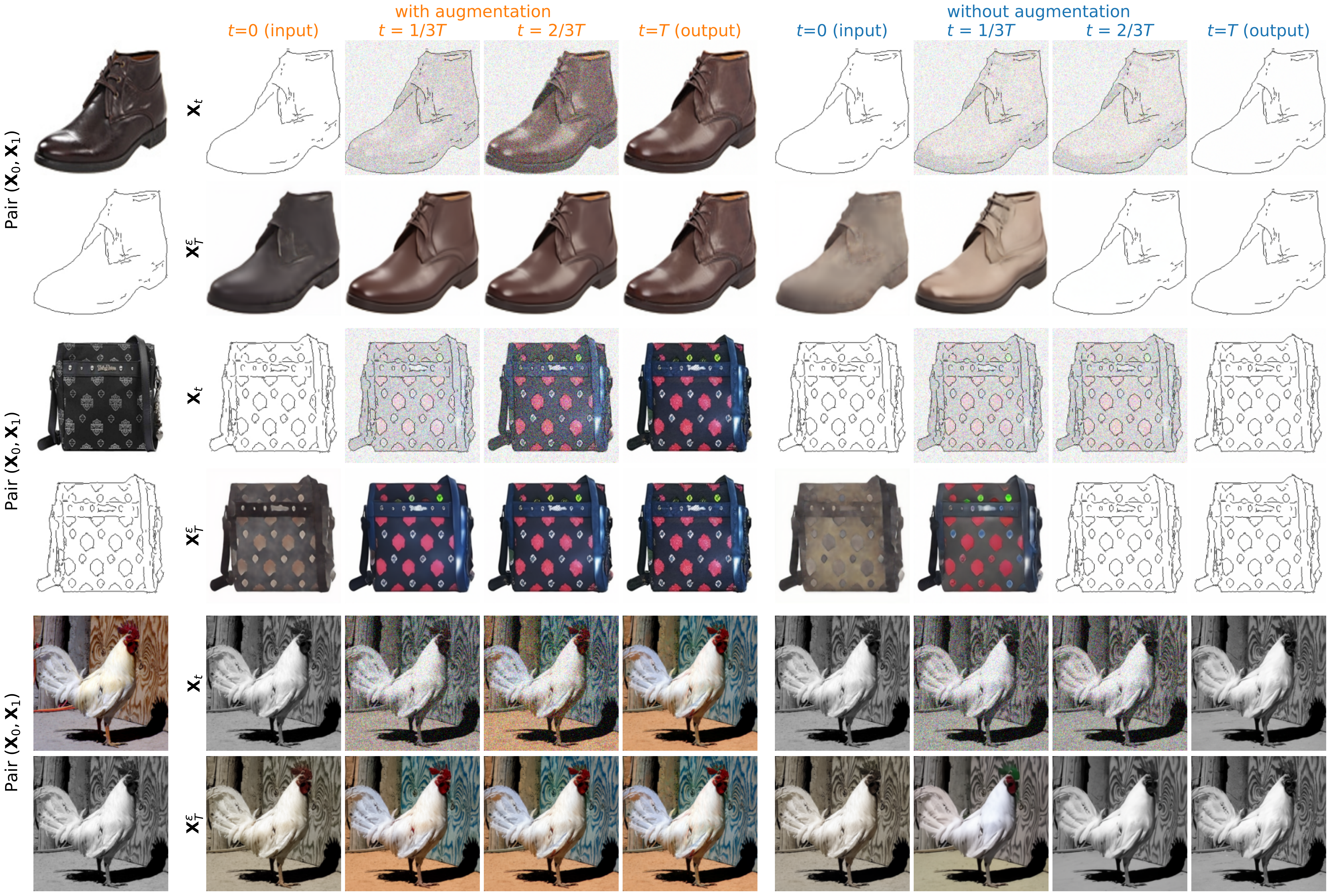}
     \caption{
        Qualitative results between our augmented bridge matching (2$^\text{nd}$-5$^\text{th}$ columns) and standard bridge matching (6$^\text{th}$-9$^\text{th}$ columns).
        We include both the generation trajectories, $\bfX_t$, and the predicted coupling, $\bfX_T^\varepsilon$, at different time steps. 
    }
     \label{fig:pix2pix}
     \vskip 0.1in
     \includegraphics[width=.6\linewidth]{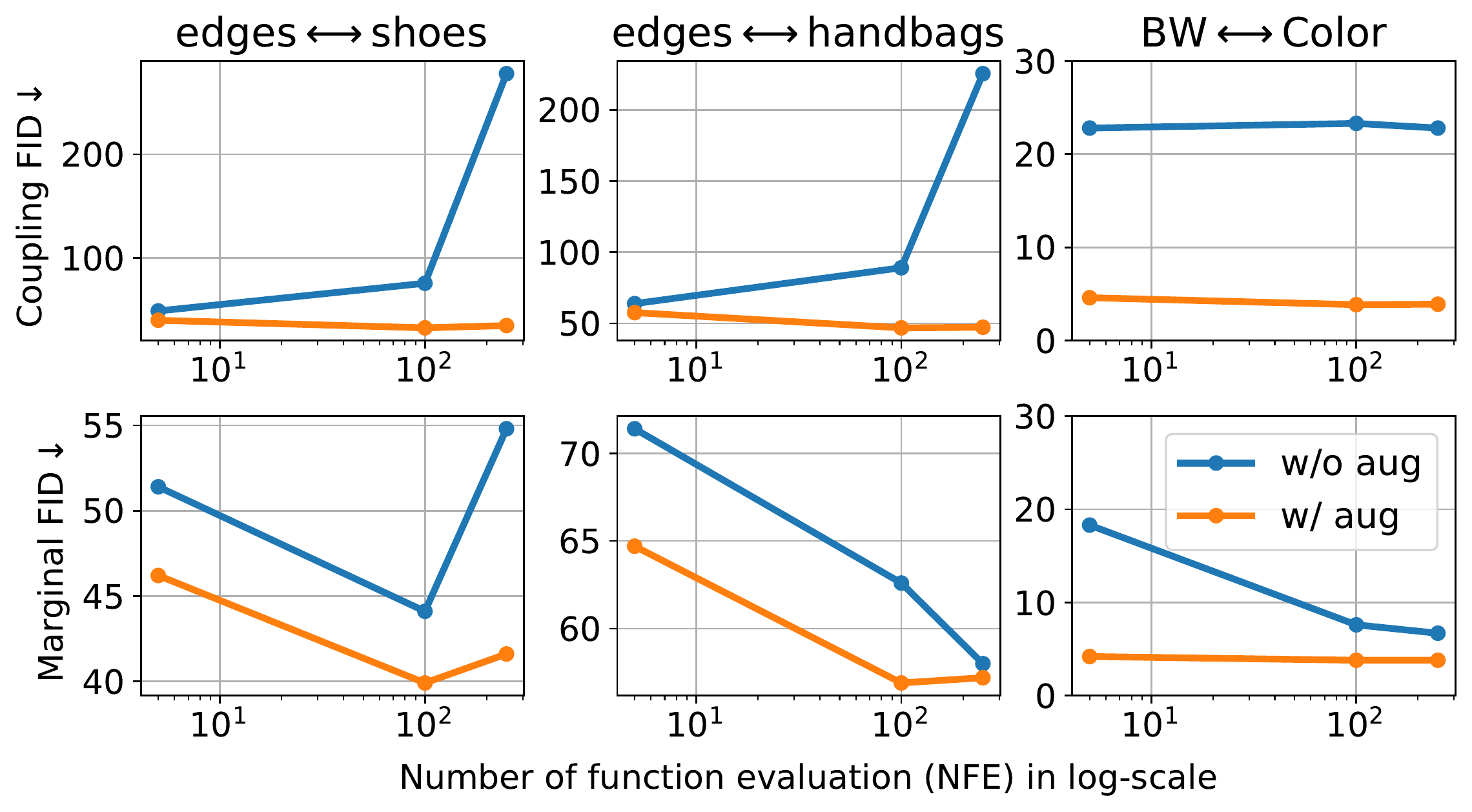}
     \caption{
        Quantitative results between standard bridge matching (marked blue) and our augmented bridge matching (marked orange). We report three pix2pix \citep{isola2017image} datasets in each column. It is clear that our augmented bridge matching not only preserve the coupling (first row), but also yields better marginal (second row).
    }
     \label{fig:pix2pix_fid}
     \vskip -0.1in
   \end{figure}

\section{Discussion}
\label{sec:conclusion}

In this work, we introduce Augmented Bridge Matching, a simple modification of the original Bridge Matching methodology which relaxes the Markovian property and preserves the original training coupling. One of our main contribution is to clarify what properties are preserved by the bridge matching procedure. In particular, we show that the training coupling is preserved if and only if it is given by the optimal transport coupling. Our main conclusion is that if the \textit{pairing} properties of the training dataset are of importance then one should use \textit{augmented} bridge matching.

Several methodological challenges remain to be addressed. First, while we have proposed a smooth interpolation between the original bridge matching and the augmented bridge matching, i.e. by conditioning on $\bfX_{\alpha t}$ with $\alpha \in \ccint{0,1}$, it is not clear which properties are preserved if $\alpha \in (0,1)$. In practice the loss of the Markovian property and the extra conditioning might make the training more difficult by increasing the variance of the loss. This is especially true if the training coupling is highly entropic, as illustrated in \Cref{sec:entropic_coupling}. In future work, we will further investigate the interaction between the loss variance and the entropy of the training coupling.

\section*{Acknowledgements}
  GHL, TC, ET are supported by ARO Award \#W911NF2010151 and DoD Basic Research Office Award HQ00342110002. VDB would like to thank Arnaud Doucet and James Thornton for their valuable advice and suggestions.

\bibliographystyle{iclr2024_conference}
\bibliography{bibliography}

\appendix

\section*{Organization of the appendix}

In \Cref{sec:basics-stoch-calc}, we recall some basics of stochastic calculus. In \Cref{sec:proof_comp}, we prove \Cref{prop:comp-with-bridge}. In \Cref{sec:proof-sde-repr}, we prove \Cref{prop:sde-repr-mixt}. We complement \Cref{prop:comp-with-bridge} with an in-depth study of the Gaussian case in \Cref{sec:proof-crefpr-example}. Additional image examples are presented in  \Cref{sec:more-images}.

\section{Basics of stochastic calculus}
\label{sec:basics-stoch-calc}

We start by recalling some useful definitions from stochastic calculus.

\begin{definition}[Infinitesimal generator]
  An operator $\mathcal{A}$ is an infinitesimal generator if there exists $b$
  measurable such that for any $\varphi \in \rmc_c^2(\rset^d)$,
  $t \in \ccint{0,1}$ and $x \in \rset^d$
  \begin{equation}
    \label{eq:infinitesimal_generator}
    \mathcal{A}_t(\varphi)(x) = \langle b_t(x), \nabla \varphi(x) \rangle + (1/2) \Delta \varphi(x) . 
  \end{equation}
\end{definition}

Equipped with the definition of infinitesimal generator, we define solutions to
martingale problem.

\begin{definition}[Martingale solution]
  We say that a path measure $\Pbb \in \mathcal{P}(\rmc(\ccint{0,1}, \rset^d))$
  is a \emph{solution to a martingale problem} with infinitesimal generator
  $\mathcal{A}$ if for any $\varphi \in \rmc_c^2(\rset^d)$,
  $u(\bfX_t) - \int_0^t \mathcal{A}_s(\varphi)(\bfX_s) \rmd s$ is a local
  $\Pbb$-martingale. 
\end{definition}

In \Cref{sec:background}, we say that
$\Pbb \in \mathcal{P}(\rmc(\ccint{0,1}, \rset^d)$ is associated with
$\rmd \bfX_t = b_t(\bfX_t) \rmd t + \rmd \bfB_t$ if $\Pbb$ is a solution to the
martingale problem with infinitesimal generator
\eqref{eq:infinitesimal_generator}. We refer to \citet{stroock2007multidimensional} for more details about the notion of martingale problem.
 
\section{Proof of \Cref{prop:comp-with-bridge}}
\label{sec:proof_comp}


We recall that $\Pbb = \Pi_{0,1} \Qbb_{|0,1}$, where
$\Qbb_{|0,1}$ is a given bridge, usually a Brownian bridge. In what follows, we
assume that $\Qbb$ is associated with
$\rmd \bfX_t = b_t(\bfX_t) \rmd t + \sigma_t \rmd \bfB_t$ with
$b: \ \ccint{0,1} \times \rset^d \to \rset^d$ and
$\sigma: \ [0,1] \to (0,+\infty)$. We recover the Brownian case if $b=0$ and
$\sigma_t = \sigma > 0$. We recall that the Markovian projection $\Pbb^{\mathcal{M}}$ is associated with 
\begin{equation}
\label{eq:markovian_projection_appendix}
      \textstyle \rmd \bfX_t^{\mathcal{M}} = \{ b_t(\bfX_t^{\mathcal{M}}) + \sigma_t^2 \mathbb{E}_{\Pbb_{1|t}}[ \nabla \log \Qbb_{1|t}(\bfX_1^{\mathcal{M}}|\bfX_t^{\mathcal{M}}) \ | \ \bfX_t^{\mathcal{M}}] \} \rmd t + \sigma_t \rmd \bfB_t , \quad \bfX_0^{\mathcal{M}} \sim \Pi_0 .
    \end{equation}

We also consider the following assumptions.

\begin{assumption}
  \label{assum:nice_diffusion}
  $(t, x_t) \mapsto \mathbb{E}_{\Pi_{1|t}}[\nabla \log
    \Qbb_{1|t}(\bfX_1|\bfX_t) \ | \ \bfX_t=x_t]$ are locally Lipschitz and there
  exist $C > 0 $, $\psi \in \rmC(\ccint{0,1}, \rset_+)$ such that for any
  $t \in \ccint{0,1}$ and $x_0, x_t \in \rset^d$, we have
 \begin{align}
        &\normLigne{\mathbb{E}_{\Pi_{1|t}}[\nabla \log
    \Qbb_{1|t}(\bfX_1|\bfX_t) \ | \ \bfX_t=x_t]} \leq C\psi(t)(1 + \normLigne{x_t}) .
  \end{align}
\end{assumption}

In addition, we consider the following assumption ensuring that the Doob $h$-transform is well-defined.

\begin{assumption}
  \label{assum:good_function_palmowski}
  For any $x_0 \in \rset^d$, $\Pi_{1|0}$ is absolutely continuous
  w.r.t. $\Qbb_{1|0}$.  For any $x_0 \in \rset^d$, let $\varphi_{1|0}$ be given
  for any $x_1 \in \rset^d$ by
  $\varphi_{1|0}(x_1|x_0) = \rmd \Pi_{1|0}(x_1|x_0) / \rmd
  \Qbb_{1|0}(x_1|x_0)$ and assume that for any $x_0 \in \rset^d$,
  $x_1 \mapsto \varphi_{1|0}(x_1|x_0)$ is bounded. For any $x_0 \in \rset^d$,
  let $\varphi_{t|0}$ given for any $x_t \in \rset^d$ and $t \in \ccint{0,1}$ by
  \begin{equation}
    \label{eq:def_h_transform}
    \textstyle
    \varphi_{t|0}(x_t|x_0) = \int_{\rset^d} \varphi_{1|0}(x_1|x_0) \rmd \Qbb_{1|t} (x_1|x_t) .
  \end{equation}
  Finally, we assume that for any $x_0 \in \rset^d$,
  $(t, x_t) \mapsto 1/\varphi_{t|0}(x_t|x_0)$ and
  $(t, x_t) \mapsto (1/2) \Delta \varphi_{t|0}(x_t|x_0)$ are bounded.
\end{assumption}

Finally, we assume that the growth of the Doob $h$-transform is controlled.

\begin{assumption}
  \label{assum:doob_h_defined}
  For any $x_0 \in \rset^d$, there exists $C \geq 0$ such that for
  any $t \in \ccint{0,1}$ and $x_t \in \rset^d$,
  $\normLigne{\nabla \log \varphi_{t|0}(x_t|x_0)} \leq C(1 +\normLigne{x_0}+\normLigne{x_t})$. 
\end{assumption}

These assumptions will allow us to apply \cite[Lemma 6]{shi2023diffusion}. However, we emphasize that these conditions are not tight and could be improved. We now state our main results.

  \begin{proposition}
    \label{prop:comp-with-bridge-appendix}
    Assume that $\Qbb$ is a Brownian motion. Assume \emph{\Cref{assum:nice_diffusion}}, \emph{\Cref{assum:good_function_palmowski}} and  \emph{\Cref{assum:doob_h_defined}}
    . In addition, assume that
    $\mathrm{H}(\Pi_i) \in \ooint{-\infty,+\infty}$ and that
    $\int_{\rset^d} \normLigne{x}^2 \rmd \Pi_i(x) < +\infty$ for
    $i \in \{0,1\}$. Let $\Pi_{0,1}^\star$ be the \emph{static} Schr\"odinger
    Bridge \begin{equation}
      \label{eq:sb_static_appendix}
    \Pi_{0,1}^\star = \argmin \ensembleLigne{\KLLigne{\Pbb_{0,1}}{\Qbb_{0,1}}}{\Pbb \in \mathcal{P}(\rset^d\times\rset^d), \ \Pbb_0=\Pi_0, \ \Pbb_1 = \Pi_1} . 
  \end{equation}Then,
  we have that $\Pbb^{\mathcal{M}}$ associated with
  \eqref{eq:markovian_projection_appendix} satisfies
  $\Pbb^\mathcal{M}_{0,1} = \Pbb_{0,1} \Qbb_{|0,1}$ if and only if
  $\Pbb_{0,1}=\Pi_{0,1}^\star$.
  \end{proposition}
  \begin{proof}
    First, assume that $\Pbb^{\mathcal{M}}_{0,1} = \Pbb_{0,1}$, i.e. the coupling is preserved. 
    Using \emph{\Cref{assum:nice_diffusion}}, \emph{\Cref{assum:good_function_palmowski}} and  \emph{\Cref{assum:doob_h_defined}}, we can apply \cite[Lemma 6]{shi2023diffusion} and we have the following result
    \begin{equation}
    \label{eq:projection_markov_app}
        \KLLigne{\Pbb}{\Pbb^\star} = \KLLigne{\Pbb}{\Pbb^{\mathcal{M}}} + \KLLigne{\Pbb^{\mathcal{M}}}{\Pbb^\star} ,
    \end{equation}
    where $\Pbb^\star$ is the Schr\"odinger bridge path measure. In particular, we have that $\Pbb^\star = \Pi^\star \Qbb_{|0,1}$. Therefore, we have that $\KLLigne{\Pbb}{\Pbb^\star} = \KLLigne{\Pbb_{0,1}}{\Pbb^\star_{0,1}}$. In addition, using the data processing inequality and that $\Pbb^\mathcal{M}_{0,1} = \Pbb_{0,1}$, we have that $\KLLigne{\Pbb^{\mathcal{M}}}{\Pbb^\star} \geq \KLLigne{\Pbb_{0,1}}{\Pbb^\star_{0,1}}$. Therefore, combining these results and \eqref{eq:projection_markov_app}, we have that 
    \begin{equation}
        \KLLigne{\Pbb}{\Pbb^\star} = \KLLigne{\Pbb_{0,1}}{\Pbb^\star_{0,1}} \geq \KLLigne{\Pbb}{\Pbb^{\mathcal{M}}} + \KLLigne{\Pbb_{0,1}}{\Pbb^\star_{0,1}} .
    \end{equation} Hence, we have $\Pbb = \Pbb^{\mathcal{M}}$. 
    In addition, we have
    that $\Pbb^{\mathcal{M}}$ is a Markov path measure,
    $\Pbb^{\mathcal{M}}_{|0,1} = \Qbb_{|0,1}$ and $\Pbb_0 = \Pi_0$,
    $\Pbb_1 = \Pi_1$. Combining \cite[Theorem
    2.14]{leonard2014reciprocal} and \cite[Theorem 2.12]{leonard2014survey}, we
    get that $\Pbb$ is the (unique) Schr\"odinger Bridge, which concludes the first part of the 
    proof. Note that the conditions of \cite[Theorem
    2.14]{leonard2014reciprocal} and \cite[Theorem 2.12]{leonard2014survey} are
    met since $\Qbb$ is a Brownian motion with $\Qbb_0 = \mathrm{Leb}$. Second,
    if $\Pi_{0,1}$ is the \emph{static} Schr\"odinger bridge, then
    $\Pi_{0,1} \Qbb_{|0,1}$ is the \emph{dynamic} Schr\"odinger bridge and hence
    Markov using \cite[Proposition 2.10]{leonard2014survey}. We conclude using \cite[Lemma
    6]{shi2023diffusion} that $\Pbb^{\mathcal{M}} = \Pbb$. 
  \end{proof}

  Let us briefly comment on the assumptions of \Cref{prop:comp-with-bridge}. The
  conditions $\mathrm{H}(\Pi_i) \in \ooint{-\infty,+\infty}$ and
  $\int_{\rset^d} \normLigne{x}^2 \rmd \Pi_i(x) < +\infty$ are assumptions on
  the marginals of the target coupling. The condition
  $\Pbb^\star = \Pi_{0,1}^\star \Qbb_{|0,1} \in \mathcal{M}$ is technical. It
  ensures that not only the Schr\"odinger bridge is Markov, which is known using
  \cite[Theorem 2.14]{leonard2014reciprocal} and \cite[Theorem
  2.12]{leonard2014survey}, but that it admits a SDE representation. This can
  also be ensured by imposing further conditions on the marginals, see
  \cite[Theorem 4.12]{leonard2011stochastic}. 

\section{Proof of \Cref{prop:sde-repr-mixt}}
\label{sec:proof-sde-repr}

  First, we recall that $\Pbb$ is given by $\Pbb = \Qbb \varphi_{0,1}$ with
  $\varphi_{0,1} = \tfrac{\rmd \Pi_{0,1}}{\rmd \Qbb_{0,1}}$. In particular, we
  have $\Pbb_{|0} = \Qbb_{|0} \varphi_{1|0}$, where
  $\varphi_{1|0} = \tfrac{\rmd \Pi_{1|0}}{\rmd \Qbb_{1|0}}$. Therefore, using \cite[Lemma 3.1, Lemma
  4.1]{palmowski2002technique}, the remark following \cite[Lemma
  4.1]{palmowski2002technique}, \tref{assum:nice_diffusion},
  \tref{assum:good_function_palmowski} and \tref{assum:doob_h_defined}, we get
  that $\Pi_{1|0}$ is Markov and associated with the distribution of
  $(\bfX_t)_{t \in \ccint{0,1}}$ given for any $t \in \ccint{0,1}$ by 
  \begin{equation}
    \label{eq:non_markov}
    \textstyle \bfX_t = \int_0^t    \nabla \log \varphi_{s|0}(\bfX_s|\bfX_0)  \rmd s + \bfB_t ,
  \end{equation}
  where for any $t \in \ccint{0,1}$, $x_0, x_t \in \rset^d$ we recall that 
  \begin{equation}
    \label{eq:varphi_int_def}
    \textstyle \varphi_{t|0}(x_t|x_0) = \int_{\rset^d} \varphi_{1|0}(x_1|x_0) \rmd \Qbb_{1|t}(x_1|x_t) . 
  \end{equation}
  First, we have that for any $t \in \ccint{0,1}$, $x_t, x_0 \in \rset^d$
  \begin{equation}
    \textstyle \Qbb_{t|0}(x_t|x_0) \varphi_{t|0}(x_t|x_0) = \int_{\rset^d} \Qbb_{t|0,1}(x_t|x_T,x_0) \rmd \Pi_{1|0}(x_1|x_0) = \Pbb_{t|0}(x_t|x_0) . 
  \end{equation}
  Therefore, we get that for any $t \in \ccint{0,1}$ and $x_t, x_0 \in \rset^d$
  \begin{equation}
    \label{eq:varphi_integrate}
    \varphi_{t|0}(x_t|x_0) = \tfrac{\rmd \Pbb_{t|0}(x_t|x_0)}{\rmd \Qbb_{t|0}(x_t|x_0)} . 
  \end{equation}
    In addition, we have the following identity for any $t \in \ccint{0,1}$, $x_0, x_t, x_1 \in \rset^d$
  \begin{equation}
    \Qbb_{1|0}(x_1|x_0) \Qbb_{t|0,1}(x_t|x_0,x_1) = \Qbb_{t|0}(x_t|x_0) \Qbb_{1|t}(x_1|x_t) .
  \end{equation}
  Using \eqref{eq:varphi_int_def}, this result and \eqref{eq:varphi_integrate},
  we get that for any $t \in \ccint{0,1}$ and $x_0,x_t \in \rset^d$
  \begin{align}
    \textstyle \nabla \log \varphi_{t|0}(x_t|x_0) &\textstyle= \int_{\rset^d} \tfrac{\Pbb_{1|0}(x_1|x_0) \Qbb_{1|t}(x_1|x_t)}{\Qbb_{1|0}(x_1|x_0) \varphi_{t|0}(x_t|x_0)} \nabla \log \Qbb_{1|t}(x_1|x_t) \rmd x_1 \\
                                                  &\textstyle= \int_{\rset^d} \tfrac{\Pbb_{1|0}(x_1|x_0) \Qbb_{t|0,1}(x_t|x_0,x_1)}{\Qbb_{t|0}(x_t|x_0) \varphi_{t|0}(x_t|x_0)} \nabla \log \Qbb_{1|t}(x_1|x_t) \rmd x_1 \\
    &\textstyle= \int_{\rset^d} \tfrac{\Pbb_{t,1|0}(x_t,x_1|x_0)}{\Pbb_{t|0}(x_t|x_0)} \nabla \log \Qbb_{1|t}(x_1|x_t) \rmd x_1 \\    
                                                  &= \textstyle  \int_{\rset^d} \nabla \log \Qbb_{1|t}(x_1|x_t) \rmd \Pbb_{1|t,0}(x_1|x_t,x_0) . 
  \end{align}
  Hence, combining this result and \eqref{eq:non_markov}, we get 
  \begin{equation}
    \label{eq:non_markov_2}
    \textstyle \bfX_t = \int_0^t \mathbb{E}_{\Pi_{1|t, 0}}[\nabla \log \Qbb_{1|t}(\bfX_1|\bfX_t) \ | \ \bfX_t, \bfX_0]  \rmd s + \bfB_t . 
  \end{equation}
  Let $\Mbb$ be Markov defined by
  $\rmd \bfX_t = + v_t(\bfX_t)  \rmd t + \rmd \bfB_t$,
  such that $\KLLigne{\Pi}{\Mbb} < +\infty$ with $v$ locally
  Lipschitz. Using \cite[Theorem 2.3]{leonard2012girsanov}, we get that
  \begin{equation}
    \textstyle \KLLigne{\Pi}{\Mbb} = \tfrac{1}{2} \int_0^1 \mathbb{E}_{\Pi_{0,t}}[\| \mathbb{E}_{\Pi_{1|t, 0}}[\nabla \log \Qbb_{1|t}(\bfX_1|\bfX_t)\ | \ \bfX_t, \bfX_0] - v_t(\bfX_t) \|^2] \rmd t . 
  \end{equation}
  In addition, we have that for any $t \in \ccint{0,1}$,
  \begin{align}
    &\mathbb{E}_{\Pbb_{0,t}}[\| \mathbb{E}_{\Pbb_{1|t, 0}}[\nabla \log \Qbb_{1|t}(\bfX_1|\bfX_t) \ | \ \bfX_t, \bfX_0]  - v_t(\bfX_t)\|^2] \\
    & \geq \mathbb{E}_{\Pbb_{0,t}}[\| \mathbb{E}_{\Pbb_{1|t, 0}}[\nabla \log \Qbb_{1|t}(\bfX_1|\bfX_t) \ | \ \bfX_t, \bfX_0]  - v^\star_t(\bfX_t) \|^2] ,
  \end{align}
  where
$v^\star_t(x_t) = \sigma_t^2 \CPELigne{\Pi_{1|t}}{\nabla \log
    \Qbb_{1|t}(\bfX_1|\bfX_t)}{\bfX_t=x_t}$ which concludes the first part of the proof.

\section{A study of the Gaussian case}
\label{sec:proof-crefpr-example}

In what follows, we illustrate
  \Cref{prop:comp-with-bridge} in a simple Gaussian setting. We consider the
  case where $\Pi_0 = \Pi_1 = \mathrm{N}(0,1)$. We assume that
  $\Pbb_{0,1} = \mathrm{N}(0, \Sigma^\alpha)$, with
  $\Sigma^\alpha \in \rset^{d \times d}$ symmetric with
  $\Sigma^\alpha_{0,0} = \Sigma^\alpha_{1,1} = 1$ and
  $\Sigma^\alpha_{0,1} = \alpha$ with $\alpha \in \ooint{0,1}$. We consider
  $\Qbb$ associated with $(\sigma \bfB_t)_{t \in \ccint{0,1}}$. We have the
  following result.

  \begin{proposition}
    \label{prop:gaussian_example_appendix}
    The static Schr\"odinger Bridge $\Pi_{0,1}^\star$, solution of
    \eqref{eq:sb_static} is given by $\Pi^\star_{0,1} = \mathrm{N}(0, \Sigma^\star)$
    with $\Sigma^\star = \Sigma^{\alpha^\star}$ with $\alpha^\star = (\sigma^2/2)(\sqrt{1+4/\sigma^4} - 1)$.
    In addition, for any $\alpha \in \ooint{0,1}$ and
    $\Pi_{0,1} = \mathrm{N}(0, \Sigma^\alpha)$, we get that
    $\Pbb_{0,1}^{\mathcal{M}} = \mathrm{N}(0, \Sigma^{f(\alpha)})$ with $f$ explicit.
   \end{proposition}

   We recall that in \Cref{fig:gaussian_example}, we empirically illustrate that
   $f(\alpha) = \alpha$, i.e. $\Pi_{0,1} = \Pbb_{0,1}^{\mathcal{M}}$ only if
   $\alpha = \alpha^\star$ for different values of $\sigma$. This confirms that,
   even in this simplified Gaussian setting, the Markovian projection only
   preserves the coupling if and only if $\Pi_{0,1}$ is the Schr\"odinger
   Bridge. 

We recall the following useful lemma.

\begin{lemma}
  \label{lemma:kl_gaussian}
  Let $\eta_i = \mathrm{N}(\mu_i, \Sigma_i)$ for $i \in \{0,1\}$, with
  $\mu_i \in \rset^d$ and $\Sigma_i$ a $d\times d$ symmetric positive
  matrices. Then, we have
  \begin{equation}
    \KLLigne{\eta_0}{\eta_1} = (-1/2)\log(\det(\Sigma_1^{-1}\Sigma_0)) + (1/2) \mathrm{Tr}(\Sigma_1^{-1}\Sigma_0) - d/2 + (1/2) \langle \mu_0 - \mu_1, \Sigma_1^{-1}(\mu_0 - \mu_1) \rangle . 
  \end{equation}
\end{lemma}

We also recall the following lemma.

\begin{lemma}
  \label{lemma:gaussian_conditional}
  Let $\eta = \mathrm{N}(\mu, \Sigma)$ with $\mu = (\mu_0, \mu_1)$ and
  \begin{equation}
    \Sigma = \left(
      \begin{matrix}
        \Sigma_{0,0} & \Sigma_{0,1} \\ \Sigma_{1,0} & \Sigma_{1,1}
      \end{matrix}
    \right) ,
  \end{equation}
  with $\mu_i \in \rset^d$ and $\Sigma_{i,j} \in \rset^{d \times d}$ for
  $i, j \in \{0, 1\}$, with $\Sigma_{0,0}$ invertible. Then for any
  $x_0 \in \rset^d$, $\eta_{1|0}(\cdot|x_0)$ is a Gaussian random variable with
  mean $\mu_0 + \Sigma_{1,0}\Sigma_{0,0}^{-1}(x_0 - \mu_0)$ and covariance matrix $\Sigma_{1,1} - \Sigma_{1,0}\Sigma_{0,0}^{-1} \Sigma_{0,1}$.
\end{lemma}

We divide the proof into two parts.
\begin{enumerate}[label=(\alph*), wide]
\item First, we derive the static Schr\"odinger Bridge. Recall that
  $\Pi_0 = \Pi_1 = \mathrm{N}(0, 1)$ and that $\Qbb$ is associated with
  $(\sigma \bfB_t)_{t \in \ccint{0,1}}$ and $\Qbb_0 = \mathrm{N}(0, 1)$. We
  recall the static problem
    \begin{equation}
      \label{eq:sb_static_appendix}
    \Pbb^\star_{0,1} = \argmin \ensembleLigne{\KLLigne{\Pbb_{0,1}}{\Qbb_{0,1}}}{\Pbb \in \mathcal{P}(\rset^d\times\rset^d), \ \Pbb_0=\Pi_0, \ \Pbb_1 = \Pi_1} . 
  \end{equation}
  Using \cite[Appendix G.1]{debortoli2021diffusion} for instance, we have
  that $\Pbb^\star$ is Gaussian with zero mean. Hence, there exists
  $\alpha^\star \in \rset$ with $\alpha^\star \in \ooint{0,1}$ such that
  \begin{equation}
    \Pbb^\star_{0,1} = \mathrm{N}(\mu^\star, \Sigma^\star) , \qquad \mu^\star = 
    \left(\begin{matrix}
      0 \\ 0 
          \end{matrix}\right), \quad \Sigma^\star = \left(
        \begin{matrix}
          1 & \alpha^\star \\
          \alpha^\star & 1
        \end{matrix}\right)
       . 
    \end{equation}
    Denote $\Pbb_{0,1}^\alpha$ such that for any $\alpha \in \ooint{0,1}$
  \begin{equation}
    \Pbb^\alpha_{0,1} = \mathrm{N}(\mu^\alpha, \Sigma^\alpha), \qquad \mu^\alpha = 
    \left(\begin{matrix}
      0 \\ 0 
          \end{matrix}\right), \quad \Sigma^\alpha = \left(
        \begin{matrix}
          1 & \alpha \\
          \alpha & 1
        \end{matrix}\right)
       . 
    \end{equation}
    We also have that
  \begin{equation}
    \Qbb_{0,1} = \mathrm{N}(\mu, \Sigma) , \qquad \mu = 
    \left(\begin{matrix}
      0 \\ 0 
          \end{matrix}\right), \quad \Sigma = \left(
        \begin{matrix}
          1 & 1 \\
          1 & 1+\sigma^2
        \end{matrix}\right)
       . 
     \end{equation}
     In particular, we have that
     \begin{equation}
\Sigma^{-1} \Sigma^\alpha = (1/\sigma^2) \left(
  \begin{matrix}
    1 + \sigma^2 - \alpha & \alpha(1+\sigma^2) - 1 \\
    -1 + \alpha & 1 - \alpha 
  \end{matrix}
\right) .        
\end{equation}
Using \Cref{lemma:kl_gaussian}, we get that
\begin{equation}
  \mathrm{Tr}(\Sigma^{-1} \Sigma^\alpha) = (1/\sigma^2)(2-2\alpha + \sigma^2) , \qquad \det(\Sigma^{-1} \Sigma^\alpha) = (1-\alpha^2)/\sigma^2 . 
\end{equation}
We have that for any $\alpha \in \ooint{0,1}$,
\begin{equation}
  \KLLigne{\Pbb_{0,1}^\alpha}{\Qbb_{0,1}} = -(1/2)\log(1-\alpha^2) - \alpha/\sigma^2 + C ,
\end{equation}
with $C$ which does not depend on $\alpha$.
Hence, at optimality, we have $1 - (\alpha^\star)^2 = \sigma^2 \alpha^\star$.
Therefore, we get that
\begin{equation}
      \Pbb^\star_{0,1} = \mathrm{N}(\mu^\star, \Sigma^\star) , \qquad \mu^\star = 
    \left(\begin{matrix}
      0 \\ 0 
          \end{matrix}\right), \quad \Sigma^\star = \left(
        \begin{matrix}
          1 & \alpha^\star \\
          \alpha^\star & 1
        \end{matrix}\right)
       , \qquad \alpha^\star = (\sigma^2/2)((1+4/\sigma^4)^{1/2} - 1) \in \ooint{0,1} . 
     \end{equation}
     In particular, note that $\alpha^\star \to 1$ if $\sigma \to 0$ and
     $\alpha^\star \to 0$ if $\sigma \to +\infty$.
   \item Next, we derive $\Pbb^{\mathcal{M}}_{0,1}$. We recall that
     $\Pbb^{\mathcal{M}}_{0,1}$ is associated with
     $(\bfX_t)_{t \in \ccint{0,1}}$ such that
     $\rmd \bfX_t = (\int_{\rset} x_1 \rmd \Pbb_{1|t}(x_1|\bfX_t) - \bfX_t) / (1-t)
     \rmd t + \sigma \rmd \bfB_t$. Therefore, we first compute the mean of
     $\Pbb_{1|t}$. For any $t \in \ccint{0,1}$ we have
     \begin{equation}
       \bfX_t = t \bfX_1 + (1-t) \bfX_0 + \sigma (t(1-t))^{1/2} \bfZ , 
     \end{equation}
     with $\bfZ \sim \mathrm{N}(0,1)$ independent from $(\bfX_0, \bfX_1)$. In
     particular, we have
     \begin{align}
       \label{eq:covariance_structure}
       &\mathrm{Cov}_\Pbb(\bfX_t, \bfX_1) = t + (1-t) \alpha , \\
       &\mathrm{Var}_\Pbb(\bfX_t) = t^2 + (1-t)^2 + \sigma^2 t (1-t) + 2t(1-t)\alpha = 1 + t(1-t)(\sigma^2 + 2\alpha - 2) . 
     \end{align}
     In what follows, we denote
     $\mu_{1|t}(x_t) = \int_{\rset} x_1 \rmd \Pbb_{1|t}(x_1|x_t)$. Using \Cref{lemma:gaussian_conditional} and 
     \eqref{eq:covariance_structure}, we have for any $t \in \ooint{0,1}$ and $x_t \in \rset$
     \begin{equation}
       \mu_{1|t}(x_t) = (t + (1-t) \alpha)/(1 + t(1-t)(\sigma^2 + 2\alpha - 2)) x_t . 
     \end{equation}
     In addition, we have for any $t \in \ooint{0,1}$ and $x_t \in \rset$
     \begin{align}       
       \mu_{1|t}(x_t) - x_t &= (t + (1-t)\alpha - 1 -t(1-t) (\sigma^2 +2 \alpha -2))/(1 +t(1-t) (\sigma^2 + 2\alpha -2))x_t \\
       &= (1-t)(-1+\alpha  - t (\sigma^2 + 2\alpha -2))/(1 +t(1-t) (\sigma^2 + 2\alpha -2)) x_t . 
     \end{align}
     Hence, we get that for any $t \in \ooint{0,1}$ and $x_t \in \rset$
     \begin{equation}
       (\mu_{1|t}(x_t) - x_t)/(1-t) = (-1+\alpha  - t (\sigma^2 +2 \alpha -2))/(1 +t(1-t) (\sigma^2 + 2\alpha -2)) x_t . 
     \end{equation}
     In what follows, we denote for any $t \in \ooint{0,1}$
     \begin{equation}
       \textstyle \kappa(t) = (-1+\alpha  - t (\sigma^2 +2 \alpha -2))/(1 +t(1-t) (\sigma^2 + 2\alpha -2)) , \qquad K(t) = \int_0^t \kappa(s) \rmd s .
     \end{equation}
     Using \cite[Section 3]{barczy2013representations}, we get that
     \begin{equation}
       \textstyle \mathrm{Cov}_{\Pbb^{\mathcal{M}}}(\bfX_1, \bfX_0) = 1 + \sigma^2 \exp[2 K(1)] \int_0^1  \exp[-2 K(s)] \rmd s ,
     \end{equation}
     which concludes the proof.
\end{enumerate}
\section{Entropic Coupling}
\label{sec:entropic_coupling}

We demonstrate that our algorithm effectively preserves the coupling when the coupling strength is not excessively entropic. Figure \ref{fig:toy} illustrates the performance of our algorithm on Gaussian Mixture datasets, where $\rmX_0$ represents the desired dataset. We introduce a controlled level of noise with a standard deviation of $k$ to the target distribution, which serves as the source for sampling $\rmX_1\sim\mathrm{N}(\rmX_0,k^2)$. The augmented algorithm succeeds in maintaining the coupling when the coupling is within a manageable entropic range. However, as the pairing becomes increasingly entropic, our algorithm encounters challenges and fails to accurately recover the data distribution.
   \begin{figure}[H]
   \vskip -0.05in
     \centering
     \includegraphics[width=\linewidth]{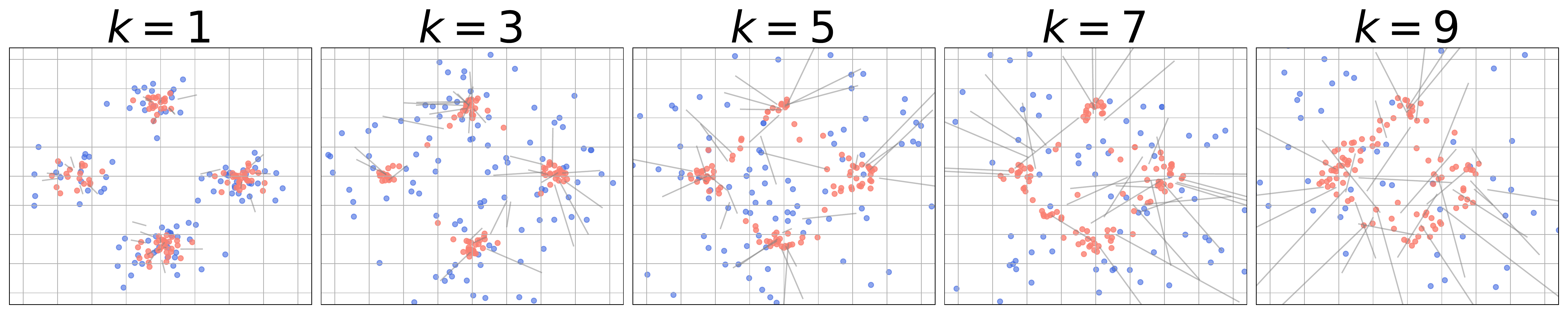}
     \caption{
        The red particles represent samples drawn from the target distribution, while the blue particles correspond to the noisy data originating from the target distribution. This noise is introduced by corrupting the target distribution with Gaussian noise, characterized by a standard deviation of $k$. The gray lines denote the pairing induced by model.
    }
     \label{fig:toy}
   \end{figure}

\section{Additional visual results}
\label{sec:more-images}

   \begin{figure}[h]
     \centering
     \includegraphics[width=\linewidth]{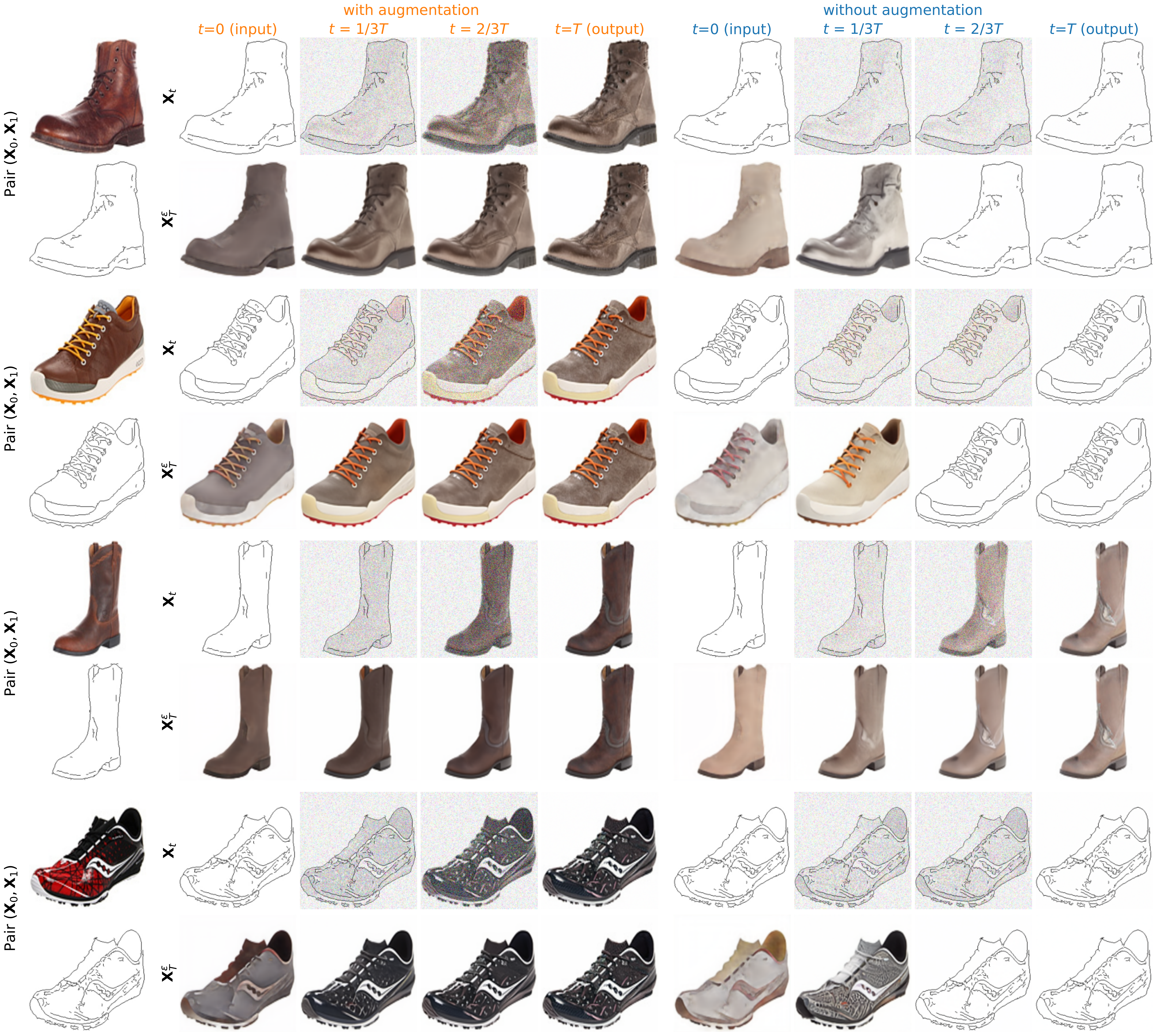}
     \caption{
        Additional qualitative results between our augmented bridge matching (2nd-5th columns) and standard bridge matching (6th-9th columns) on the bidirectional translation of \texttt{edges2shoes} task.
    }
     \label{fig:pix2pix-e2s}
   \end{figure}

   \begin{figure}[h]
     \centering
     \includegraphics[width=\linewidth]{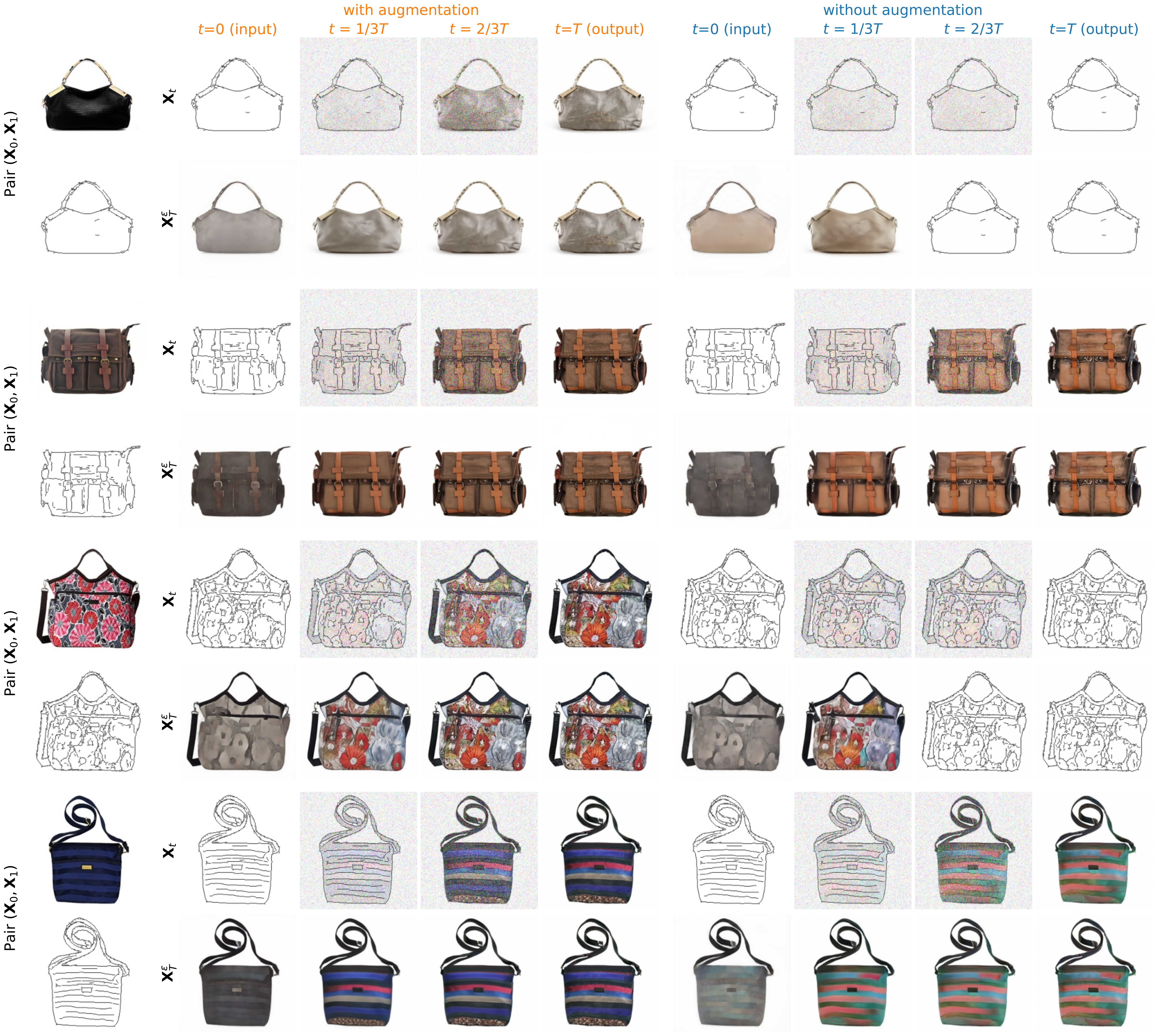}
     \caption{
        Additional qualitative results between our augmented bridge matching (2nd-5th columns) and standard bridge matching (6th-9th columns) on the bidirectional translation of \texttt{edges2handbags} task.
    }
     \label{fig:pix2pix-e2h}
   \end{figure}

   \begin{figure}[h]
     \centering
     \includegraphics[width=\linewidth]{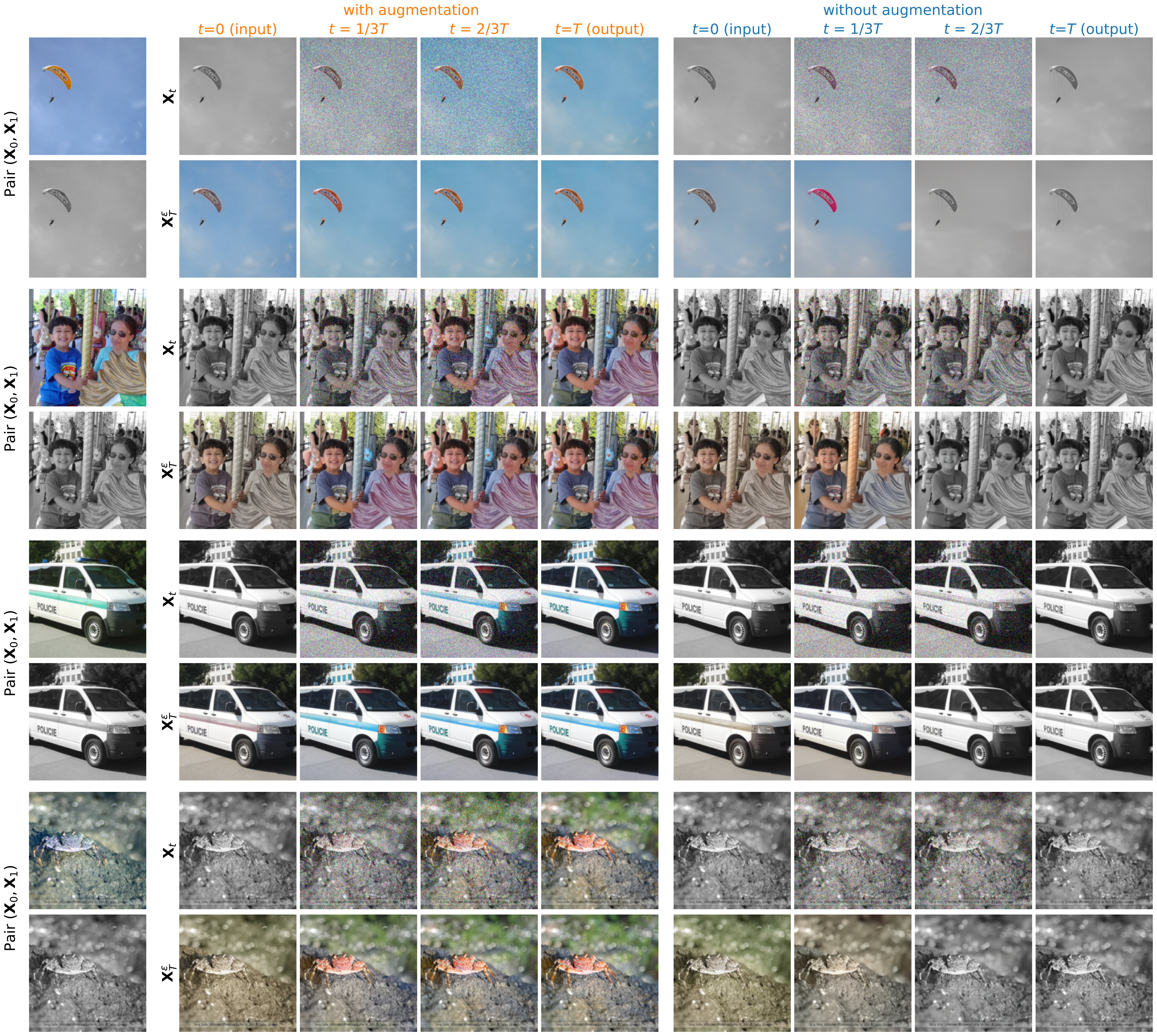}
     \caption{
        Additional qualitative results between our augmented bridge matching (2nd-5th columns) and standard bridge matching (6th-9th columns) on the bidirectional translation of (de-)colorization task.
    }
     \label{fig:pix2pix-c}
   \end{figure}

\end{document}

%% file: math_commands.tex
\usepackage{amsmath,amsfonts,bm}

\def\eqref#1{equation~\ref{#1}}

\def\1{\bm{1}}

\def\rmC{{\mathbf{C}}}

\def\rmX{{\mathbf{X}}}

\DeclareMathAlphabet{\mathsfit}{\encodingdefault}{\sfdefault}{m}{sl}
\SetMathAlphabet{\mathsfit}{bold}{\encodingdefault}{\sfdefault}{bx}{n}

\DeclareMathOperator*{\argmin}{arg\,min}

%% file: def.tex
\def\Mbb{\mathbb{M}}
\def\calM{\mathcal{M}}

\newcommand{\tref}[1]{\textup{\Cref{#1}}}

\def\sl{\mathfrak{sl}}

\newcommand{\tta}{\mathtt{A}}

\newcommand{\Capprox}{\tta}

\newcommandx\ctun[1][1=T]{\Capprox_{#1,1}}

\newcommandx{\expec}[2]{{\mathbb E}\left[#1 \middle \vert #2  \right]} 

\newcommandx{\norm}[2][1=]{\ifthenelse{\equal{#1}{}}{\left\Vert #2 \right\Vert}{\left\Vert #2 \right\Vert^{#1}}}
\newcommandx{\normLigne}[2][1=]{\ifthenelse{\equal{#1}{}}{\Vert #2 \Vert}{\Vert #2\Vert^{#1}}}

\def\bfc{\mathbf{c}}
\def\bfY{\mathbf{Y}}

\def\bfX{\mathbf{X}}

\def\bfZ{\mathbf{Z}}

\def\bfZ{\mathbf{Z}}

\def\bfB{\mathbf{B}}

\def\msx{\mathsf{X}}

\def\Qbb{\mathbb{Q}}
\def\Mbb{\mathbb{M}}
\def\Pbb{\mathbb{P}}

\def\rset{\mathbb{R}}

\def\rmd{\mathrm{d}}

\def\rmc{\mathrm{C}}
\def\rmC{\mathrm{C}}

\newcommandx{\functionspace}[2][1=+]{\mathbb{F}_{#1}(#2)}

\newcommandx{\VarDeux}[3][3=]{\operatorname{Var}^{#3}_{#1}\left\{#2 \right\}}

\newcommand{\LeftEqNo}{\let\veqno\@@leqno}

\newcommand{\N}{\ensuremath{\mathbb{N}}}

\newcommand{\PE}{\mathbb{E}}

\newcommandx{\Vnorm}[2][1=V]{\| #2 \|_{#1}}
\newcommandx{\VnormEq}[2][1=V]{\left\| #2 \right\|_{#1}}

\newcommandx\probaMarkovTilde[2][2=]
{\ifthenelse{\equal{#2}{}}{{\widetilde{\mathbb{P}}_{#1}}}{\widetilde{\mathbb{P}}_{#1}\left[ #2\right]}}

\newcommand{\expeLigne}[1]{\PE [ #1 ]}

\def\eqsp{\;}

\newcommand{\ooint}[1]{\left(#1\right)}
\newcommand{\ccint}[1]{\left[#1\right]}

\newcommandx{\weight}[2][2=n]{\omega_{#1,#2}^N}

\newcommandx\sequence[3][2=,3=]
{\ifthenelse{\equal{#3}{}}{\ensuremath{\{ #1_{#2}\}}}{\ensuremath{\{ #1_{#2}, \eqsp #2 \in #3 \}}}}

\newcommandx\sequenceD[3][2=,3=]
{\ifthenelse{\equal{#3}{}}{\ensuremath{\{ #1_{#2}\}}}{\ensuremath{( #1)_{ #2 \in #3} }}}

\newcommandx{\sequencen}[2][2=n\in\N]{\ensuremath{\{ #1_n, \eqsp #2 \}}}
\newcommandx\sequenceDouble[4][3=,4=]
{\ifthenelse{\equal{#3}{}}{\ensuremath{\{ (#1_{#3},#2_{#3}) \}}}{\ensuremath{\{  (#1_{#3},#2_{#3}), \eqsp #3 \in #4 \}}}}
\newcommandx{\sequencenDouble}[3][3=n\in\N]{\ensuremath{\{ (#1_{n},#2_{n}), \eqsp #3 \}}}

\newcommand{\opnorm}[1]{{\left\vert\kern-0.25ex\left\vert\kern-0.25ex\left\vert #1
    \right\vert\kern-0.25ex\right\vert\kern-0.25ex\right\vert}}

\def\Id{\operatorname{Id}}

\newcommandx{\CPE}[3][1=]{{\mathbb E}_{#1}\left[#2 \middle \vert #3  \right]} 
\newcommandx{\CPELigne}[3][1=]{{\mathbb E}_{#1}[#2  \vert #3  ]} 
\newcommandx{\CPEsq}[3][1=]{{\mathbb{E}^{1/2}}_{#1}\left[#2 \middle \vert #3  \right]} 
\newcommandx{\CPVar}[3][1=]{\mathrm{Var}^{#3}_{#1}\left\{ #2 \right\}}
\newcommand{\CPP}[3][]
{\ifthenelse{\equal{#1}{}}{{\mathbb P}\left(\left. #2 \, \right| #3 \right)}{{\mathbb P}_{#1}\left(\left. #2 \, \right | #3 \right)}}

\newcommandx{\osc}[2][1=]{\mathrm{osc}_{#1}(#2)}

\def\Id{\operatorname{Id}}

\newcommand{\ensembleLigne}[2]{\{#1\,:\eqsp #2\}}

\newcommand\coupling[2]{\Gamma(\mu,\nu)}

\def\Leb{\lambda}

\newcommandx{\KLLigne}[2]{\operatorname{KL}( #1 | #2 )}
\newcommandx{\KLLignesqrt}[2]{\operatorname{KL}^{1/2}( #1 | #2 )}

\def\gaStep
\def\QKer{Q}

\def\distance{\mathbf{d}}
\newcommandx{\wasserstein}[3][1=\distance,3=]{\mathbf{W}_{#1}^{#3}\left(#2\right)}
\newcommandx{\wassersteinLigne}[3][1=\distance,3=]{\mathbf{W}_{#1}^{#3}(#2)}
\newcommandx{\wassersteinD}[1][1=\distance]{\mathbf{W}_{#1}}
\newcommandx{\wassersteinDLigne}[1][1=\distance]{\mathbf{W}_{#1}}

\def\sigmaD{\sigma^2}

\newcommandx{\phibfs}[1][1=]{\pmb{\varphi}_{\sigmaD_{#1}}}

\newcommandx\sequenceg[3][2=,3=]
{\ifthenelse{\equal{#3}{}}{\ensuremath{( #1_{#2})}}{\ensuremath{( #1_{#2})_{ #2 \geq #3}}}}

\newcommandx{\distV}[1][1=\bfc]{\mathbf{W}_{#1}}
\newcommandx{\distVdeux}[1][1=W_2]{\mathbf{d}_{#1}}

\def\pathmeas{\mathcal{P}(\mathcal{C})}

